\documentclass{article}
\usepackage{ijcai25}

\usepackage{times}
\usepackage{soul}
\usepackage{url}
\usepackage[hidelinks]{hyperref}
\usepackage[utf8]{inputenc}
\usepackage[small]{caption}
\usepackage{graphicx}
\usepackage{amsmath}
\usepackage{amsthm}
\usepackage{booktabs}
\usepackage{algorithm}
\usepackage{algorithmic}
\usepackage[switch]{lineno}

\newcommand{\tablestyle}[2]{\setlength{\tabcolsep}{#1}\renewcommand{\arraystretch}{#2}\centering}

\usepackage{subcaption}
\usepackage{multirow}
\usepackage{amsmath,amssymb,mathtools,amsthm,algorithm,algorithmic,mathtools,booktabs,bm,color,natbib}

\newtheorem{definition}{Definition}
\newtheorem{lemma}{Lemma} 
\linenumbers

\newtheorem{theorem}{Theorem}

\title{Self-Consistent Model-based Adaptation for Visual Reinforcement Learning}

\author{
Xinning Zhou$^1$\thanks{The authors contributed equally to this work.}
\and
Chengyang Ying$^1$\footnotemark[1]\and
Yao Feng$^{1}$\and
Hang Su$^1$\And
Jun Zhu$^1$\\
\affiliations
$^1$Department of Computer Science \& Technology, Institute for AI, BNRist Center, Tsinghua-Bosch Joint ML Center, THBI Lab, Tsinghua University\\
\emails
zxn21@mails.tsinghua.edu.cn
}
\begin{document}
\nolinenumbers
\maketitle

\begin{abstract}

Visual reinforcement learning agents typically face serious performance declines in real-world applications caused by visual distractions. Existing methods rely on fine-tuning the policy's representations with hand-crafted augmentations. In this work, we propose Self-Consistent Model-based Adaptation (SCMA), a novel method that fosters robust adaptation without modifying the policy. By transferring cluttered observations to clean ones with a denoising model, SCMA can mitigate distractions for various policies as a plug-and-play enhancement. To optimize the denoising model in an unsupervised manner, we derive an unsupervised distribution matching objective with a theoretical analysis of its optimality. We further present a practical algorithm to optimize the objective by estimating the distribution of clean observations with a pre-trained world model. Extensive experiments on multiple visual generalization benchmarks and real robot data demonstrate that SCMA effectively boosts performance across various distractions and exhibits better sample efficiency.

\end{abstract}

% Visual reinforcement learning (VRL)has shown the potential in accomplishing complex robot control tasks through visual observations~\citep{hafner2023mastering,brohan2023can}. Despite the success in clean training environments, the unknown visual distractions in real-world testing environments may significantly hinder the agents' performance by deviating them from task-oriented goals, e.g. changing backgrounds~\citep{Hansen2020SelfSupervisedPA,fu2021learning}. Therefore, the ability to maintain similar performance under visual distractions is a prerequisite for the real-world application of RL agents.

% Despite the success in clean training environments with minimum distractions, VRL agents struggle to generalize when deployed in practical environments with unexpected visual distractions, such as changes in textures or complex backgrounds~\citep{Hansen2020SelfSupervisedPA,fu2021learning}. This discrepancy between training and deployment environments severely limits the applicability of VRL policies in diverse, real-world scenarios.

\section{Introduction}

Visual reinforcement learning (VRL) aims to complete complex tasks with high-dimensional observations, which has achieved remarkable results in various domains~\citep{hafner2019dream,brohan2023can,Li2024Think2DriveER}. Since VRL agents are typically trained on clean observations with minimal distractions, they struggle to handle cluttered observations when deployed in real-world environments with unexpected visual distractions, such as changes in textures or complex backgrounds~\citep{Hansen2020SelfSupervisedPA,fu2021learning}. The discrepancy between clean and cluttered observations results in a serious performance gap.
%The vulnerability to visual distractions severely limits the applicability of VRL policies in practical scenarios.

% \hangx{This paragraph is important. While I think clean environment is not well-defined, what is training and what is testing? what problem you are trying to solve.  }

% \hang{Visual reinforcement learning (VRL) aims to solve complex tasks using high-dimensional visual inputs, such as video games~\citep{hafner2019dream}, autonomous driving~\citep{Li2024Think2DriveER}, and robotic control~\citep{brohan2023can}. While agents excel in controlled training environments with minimal distractions, they often struggle to generalize when deployed in real-world situations with unexpected visual variations—such as changes in lighting, textures, or backgrounds~\citep{Hansen2020SelfSupervisedPA,fu2021learning}. This discrepancy between training and deployment environments limits the practical usability of VRL systems, which is crucial for ensuring VRL agents are reliable in diverse, real-world scenarios. }

The key to closing the performance gap is to make the policy invariant to distractions. Most existing methods aim to mitigate distractions by learning robust representations. In particular, one line of work is to align the policy's representation between the clean and cluttered observations. Due to the lack of paired data, prevailing methods use hand-crafted functions to create augmentations similar to cluttered observations~\citep{hansen2021generalization,Bertoin2022LookWY}. The effectiveness of such methods is typically limited in settings without prior knowledge of potential distractions. Another line of work addresses the problem through adaptation, which boosts deployment performance by fine-tuning the policy's representation with self-supervised objectives. However, existing adaptation-based methods often lead to narrow empirical increases~\citep{Hansen2020SelfSupervisedPA} or are effective only for a specific type of distractions~\citep{yang2024movie}. Moreover, the practical application of VRL often requires different policies to ensure robustness against the same types of distractions~\citep{devo2020towards}. For instance, domestic robots for different tasks all face the challenge imposed by residential backgrounds with similar distractions. Since policies trained for different tasks have distinct representations, current methods need to fine-tune them separately as the modification made to one policy's representations is not directly applicable to another policy.
To address the above issues, we propose \textbf{S}elf-\textbf{C}onsistent \textbf{M}odel-based \textbf{A}daptation (SCMA), a novel method that fosters robust adaptation for various policies as a plug-and-play enhancement. Instead of fine-tuning policies, SCMA utilizes a denoising model to mitigate distractions by transferring cluttered observations to clean ones. Therefore, the denoising model is policy-agnostic and can be seamlessly combined with any policy to boost performance under distractions without modifying its parameters. We further design an unsupervised distribution matching objective to optimize the denoising model in the absence of paired data. Theoretically, we show that the solution set of the unsupervised objective strictly contains the optimal solution in the supervised setting. The proposed objective regularizes the outputs of the denoising model to follow the distribution of observations in clean environments, which we choose to estimate with a pre-trained world model~\citep{hafner2019learning,hafner2023mastering}.

We practically evaluate the effectiveness of SCMA with the commonly adopted DMControlGB~\citep{Hansen2020SelfSupervisedPA,hansen2021generalization}, DMControlView~\citep{yang2024movie}, and RL-ViGen~\citep{yuan2024rl}, where the agent needs to complete continuous control tasks in environments with visual distractions. Extensive results show that SCMA significantly narrows the performance gap caused by various types of distractions, including natural video background, moving camera view, and occlusion. Also, we verify the effectiveness of SCMA with real-world robot data, showing its future potential in real-world deployment. In summary, the main contributions of this paper are:

% \hangx{when we summarize the contribution, better to focus on the key challenges that you are addressing. focus on your novelty rather than what you are doing}
% \junz{if the above is clear enough, not necessary to summarize again...}
\begin{itemize}    
    \item We address the challenge of visual distractions by transferring observations and derive an unsupervised distribution matching objective with theoretical analysis.
    \item We propose self-consistent model-based adaptation (SCMA), a novel method that promotes robust adaptation for different policies in a plug-and-play manner.
    \item Extensive experiments show that SCMA significantly closes the performance gap caused by various types of distractions. We also demonstrate the effectiveness of SCMA with real-world robot data.
\end{itemize}
% \junz{you use "significantly closes ... gap", "successfully closes ... gap"; this looks confusing, did we close it (means no gap) or just improve? make it more precise...}

\section{Related Work}
% \hangx{!! related work is to clarify the problem, key challenges, and the unsolved issues for the current methods}
% VRL 重要, gen重要, bisim, augmen, adapt
\subsection{Visual Generalization in RL}
% The ability to interpret visual input plays an essential role for reinforcement learning (RL) agents to solve a wide range of challenging tasks~\citep{chopra2020end,chaplot2020object,shridhar2023perceiver}. Recently, VRL has shown impressive achievements in clean environments but struggles to generalize to visually distracting environments~\citep{tomar2021learning,liu2023cross}. The key to visual generalization is the ability to neglect visual distractions, i.e. acquire insensitive representations of distractions, and 

% which require extra model capacity for modeling useless distractions. 
% which no longer requires prior knowledge of visual distractors
%  without the connection between the adaptation objectives and intended downstream tasks

The ability to generalize across environments with unknown distractions is a long-stand challenge for the practical application of reinforcement learning (RL) agents~\citep{chaplot2020object,shridhar2023perceiver,tomar2021learning,liu2023cross,ying2024peac}. Task-induced methods address the problem by learning structured representations that separate task-relevant features from confounding factors~\citep{fu2021learning,pan2022isolating,wang2022denoised}. Augmentation-based methods regularize the representation between augmented images and clean equivalents~\citep{hansen2021generalization,ha2023dream}, but they require prior knowledge of the test-time variations to manually design augmentations. Adaptation-based methods~\citep{Hansen2020SelfSupervisedPA,yang2024movie} do not assume the distractions and fine-tune the agent’s representation through self-supervised objectives. However, existing adaptation-based methods tend to lead to narrow empirical improvement~\citep{Hansen2020SelfSupervisedPA} or are limited to a specific type of visual distractions~\citep{yang2024movie}. Several studies aim to tackle this issue with foundation models~\citep{nair2022r3m,shah2023gnm}, but they still struggle with computational budget and inference time.

\subsection{Unsupervised Domain Transfer}
Unsupervised Domain Transfer aims to map data collected from the source domain to a related target domain without explicit supervision signals~\citep{wang2021survey}. The topic has been explored in various research areas, such as style transfer~\citep{Zhu2017UnpairedIT,Zhao2022EGSDEUI}, pose transfer~\citep{Li2023CrossLocoHM}, language translation~\citep{lachaux2020unsupervised,artetxe2017unsupervised} and so on. However, one key difference between our setting and theirs is that we can interact with the environments to collect data rather than relying on pre-collected static datasets. Therefore, we can obtain a certain level of control over the distribution of collected data by selecting specific action sequences, which makes it possible for us to achieve the desired transfer from cluttered observations to clean ones with unsupervised distribution matching~\citep{cao2018unsupervised,baktashmotlagh2016distribution}.

\begin{figure}
    \centering
    \includegraphics[width=0.7\linewidth]{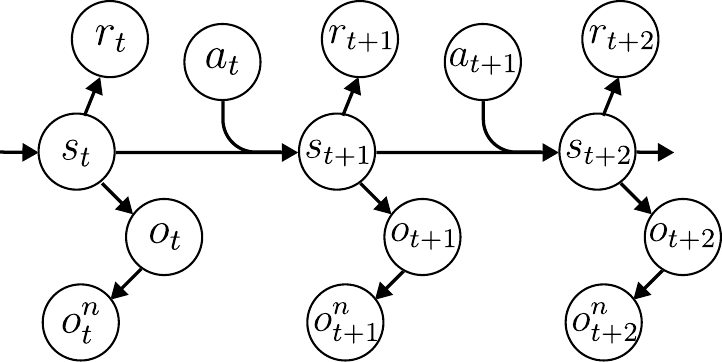}
    \caption{The graphical model of a NPOMDP, where $o_t$ and $o^n_t$ denote the clean and cluttered observation respectively.}
    \vspace{-1em}
    \label{fig_pgm}
\end{figure}

\begin{figure*}[tb]
\centering
% \centerline{\includegraphics[width=0.8\linewidth]{imgs/SCMA-algo-6.pdf}}
\centerline{\includegraphics[width=1.0\linewidth]{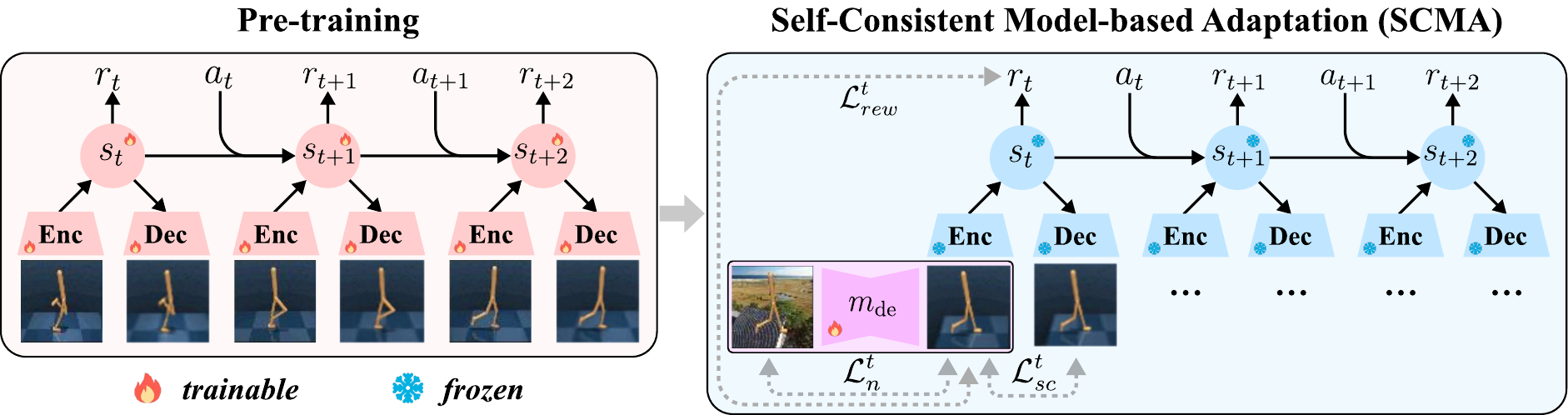}}
\caption{\textbf{An overview of Self-Consistent Model-based Adaption (SCMA)}. SCMA adapts the agent to distracting environments by transferring cluttered observations to clean ones with the denoising model $m_{\mathrm{de}}$. Leveraging a pre-trained world model, $m_{\mathrm{de}}$ can be efficiently optimized with self-consistent reconstruction, noisy reconstruction, and reward prediction loss.}
\label{fig_scma_algo}
\vspace{-1.2em}
\end{figure*}

\section{Methodology}
\label{sec_method}

We first present our problem formulation and the supervised objective $\mathcal{L}_{O}$ in Sec.~\ref{subsec_problem}. Then we introduce an unsupervised distribution matching surrogate $\mathcal{L}_{\mathrm{KL}}$ and analyze the connection between $\mathcal{L}_{\mathrm{KL}}$ and $\mathcal{L}_{O}$ in Sec.~\ref{subsec_mitigating}. Finally, we transform $\mathcal{L}_{\mathrm{KL}}$ into several optimizable adaptation losses in Sec.~\ref{subsec_adaptation}, along with practical enhancements in Sec.~\ref{subsec_homo}.

\subsection{Problem Formulation}
\label{subsec_problem}

We formalize visual RL with distractions as a Noisy Partially-Observed Markov Decision Process (NPOMDP) $\mathcal{M}_n = \langle \mathcal{S}, \mathcal{O}, \mathcal{A}, \mathcal{T}, \mathcal{R}, \gamma, \rho_0, f_n\rangle$. In a NPOMDP, $\mathcal{S}$ is the hidden state space, $\mathcal{O}$ is the discrete observation space, $\mathcal{A}$ denotes the action space, $\mathcal{T}: \mathcal{S}\times\mathcal{A}\mapsto \Delta(\mathcal{S})$ defines the transition probability distribution over the next state, $\mathcal{R}: \mathcal{S}\times\mathcal{A}\mapsto \mathbb{R}$ is the reward function, $\gamma$ is the discount factor, and $\rho_0$ is the initial state distribution. Here $f_n: \mathcal{O}\mapsto\mathcal{O}$ is a noise function that maps a clean observation $o_t$ to its cluttered version $o^{n}_t=f_n(o_t)$. Following the common settings~\citep{Hansen2020SelfSupervisedPA,Bertoin2022LookWY}, we assume that $f_n$ is injective so that the distractions do not corrupt the original information. The graphical model of NPOMDP is provised in Fig.~\ref{fig_pgm} 

Given the action sequence $a_{1:T}$, the conditional joint distribution describing the environment's latent dynamics is defined as:
\begin{equation}
\resizebox{1.0\linewidth}{!}{$
            \displaystyle
            \begin{aligned} 
                &p(o_{1:T}, o^{n}_{1:T}, r_{1:T}|a_{1:T}) \coloneqq \\
                &\int\prod\limits_{t=1}^{T} p(o^n_t|o_t) p(o_{t}|s_{\leq t}, a_{<t})p(r_{t}|s_{\leq t}, a_{<t})p(s_{t}|s_{<t},a_{<t})\mathrm{d}s_{1:T}.
            \end{aligned} 
$}
\label{eq_joint_distribution}
\end{equation}
% \junz{here we assume $f_n$ is deterministic? any explanation?} 
We denote $p(o^n_t|o_t) = \delta(o^n_t - f_n(o_t))$ as the \textit{noising distribution} of $f_n$, which is a Dirac distribution with $\delta(\cdot)$ being the Dirac delta function~\citep{dirac1981principles}. 
Leveraging the Bayes' rule, the posterior distribution $p(o_t|o^n_t)$ can also be derived from Eq.~\ref{eq_joint_distribution}, which we denote as the \textit{posterior denoising distribution} of $f_n$. 

% By transferring observations, we can mitigate distractions for different policies in a plug-and-play way.
The performance of policies pre-trained with clean observations often degenerates when handling cluttered observations~\citep{Hansen2020SelfSupervisedPA,Bertoin2022LookWY}. To fill the performance gap, a natural way is to transfer cluttered observations to their corresponding clean ones by estimating the posterior denoising distribution $p(o_t|o^n_t)$. In the supervised setting, we can estimate $p(o_t|o^n_t)$ with a learnable distribution $q(o_t|o^n_t)$ by minimizing the following objective:
\begin{equation*}
% \label{eq_denoise_objective}
\begin{aligned} 
\mathcal{L}_O & \coloneqq \mathbb{E}_{p(o_{1:T},o^n_{1:T}|a_{1:T})} \log q(o_{1:T}|o^n_{1:T})\\
&= \mathbb{E}_{p(o_{1:T},o^n_{1:T}|a_{1:T})}\sum_t \log q(o_t|o^n_t).
\end{aligned}
\end{equation*}
% \junz{what's $q$? add necessary explanation}

 We further show that $p(o_t|o^n_t)$ is a Dirac distribution when $f_n$ is injective. Therefore, we adopt a denoising model $m_{\mathrm{de}}$ and choose $q(o_t|o^n_t)=\delta(o_t-m_{\mathrm{de}}(o^n_t))$ in practice. More details can be found in Appendix~\ref{appendix_subsec_npomdp}.

\subsection{Mitigating Visual Distractions with Unsupervised Distribution Matching}
\label{subsec_mitigating}
% \junz{how difficult to get paired data? impossible or expensive?}
% \zxn{why can we achieve this}
% While we only can collect observations from clean environments (i.e. $p(o_{1:T}|a_{1:T})$) and distracting environments (i.e. $p(o^n_{1:T}|a_{1:T})$) separately, the absence of paired data between $o_t$ and $o'_t$ makes it impossible to optimize $\mathcal{L}_{\mathcal{O}}$ directly with supervised methods. Inspired by unsupervised distribution matching~\citep{baktashmotlagh2016distribution,cao2018unsupervised}, we propose an unsupervised surrogate $\mathcal{L}_{\mathrm{KL}}$ that minimizes the KL-divergence between the action-conditioned joint distribution of the following two distributions:
% \begin{equation}
% \label{eq_surrogate_kl}
% \begin{aligned}
%     \mathcal{L}_{\mathrm{KL}} =& \mathrm{D}_{\mathrm{KL}}\left (\mathbb{E}_{p(o^n_{1:T}|a_{1:T})}[q(o_{1:T}|o^n_{1:T})] \big\Vert p(o_{1:T}|a_{1:T})\right ).
% \end{aligned}
% \end{equation}

The direct optimization of $\mathcal{L}_{\mathcal{O}}$ requires collecting paired observations $(o_t,o^n_t)$. Since we can only collect observations from clean environments (i.e. $p(o_{1:T}|a_{1:T})$) and distracting environments (i.e. $p(o^n_{1:T}|a_{1:T})$) separately, the absence of paired data imposes severe challenges. Inspired by unsupervised distribution matching~\citep{baktashmotlagh2016distribution,cao2018unsupervised}, we propose to minimize the KL-divergence between the action-conditioned distribution of the clean and transferred observations, which leads to the following unsupervised surrogate $\mathcal{L}_{\mathrm{KL}}$ (see Appendix~\ref{appendix_subsec_distribution} for details):
\begin{equation*}
% \label{eq_surrogate_kl}
\begin{aligned}
    \mathcal{L}_{\mathrm{KL}} & \coloneqq \mathrm{D}_{\mathrm{KL}} \Big(p(o^n_{1:T}|a_{1:T})q(o_{1:T}|o^n_{1:T}) \\
    &\quad\quad\quad\quad\quad\quad \big\Vert p(o_{1:T}|a_{1:T})q(o^n_{1:T}|o_{1:T})\Big),
\end{aligned}
\end{equation*}
where $q(o_{1:T}|o^n_{1:T})=\prod_t q(o_t|o^n_t)$ and $q(o^n_{1:T}|o_{1:T})=\prod_t q(o_t|o^n_t)$ are learnable noisy and denoising distribution respectively.

To analyze the connection between $\mathcal{L}_{\mathrm{KL}}$ and $\mathcal{L}_{\mathcal{O}}$, we first introduce the concept of \textit{homogeneous noise functions}, which are theoretically indistinguishable in the unsupervised setting ad defined below (Details are deferred to Appendix~\ref{appendix_subsec_distribution}):
% \ycy{i.e., we can not distinguish homogeneous noise functions only with unsupervised observations $p(o_{1:T}|o^n_{1:T})$ and $p(o^n_{1:T}|o_{1:T})$}
\begin{definition}
\label{def_homo_noise}
For noise functions $f_{n_i}$, we denote $o^{n_i}_t=f_{n_i}(o_t)$ as its cluttered observation. Given the distribution of clean observation $p(o_{1:T}|a_{1:T})$, we call the noise functions $f_{n_1}$ and $f_{n_2}$ to be homogeneous under $p(o_{1:T}|a_{1:T})$ if their cluttered observations have the same distribution, i.e.:
\begin{equation*}
\begin{array}{l}
    f_{n_1} \equiv_{p} f_{n_2} \Leftrightarrow \;p(o^{n_1}_{1:T}|a_{1:T}) = p(o^{n_2}_{1:T}|a_{1:T}),\\
    \\
    \text{where } \;p(o^{n_i}_{1:T}|a_{1:T}) = \sum\limits_{o_{1:T}}p(o_{1:T}|a_{1:T})p(o^{n_i}_{1:T}|o_{1:T}).
\end{array}
\end{equation*}
We define $\mathcal{H}_{f_n}^p = \left \{ f_{n_i} | f_{n_i}\equiv_p f_{n}\right \}$, which includes all homogeneous noise functions of $f_n$ under $p(o_{1:T}|a_{1:T})$.
\end{definition}

Then we show that the solution set of $\mathcal{L}_{\mathrm{KL}}$ equals the set of posterior denoising distribution of noise functions in $\mathcal{H}^p_{f_n}$. Since $f_n$ is clearly in $\mathcal{H}^p_{f_n}$, the solution set of $\mathcal{L}_{\mathrm{KL}}$ contains $p(o_t|o^n_t)$, which is the optimal solution to $\mathcal{L}_{\mathcal{O}}$.

\begin{theorem}[Proof in Appendix~\ref{appendix_subsec_optimality}]
\label{theorem_homo}
Given $p(o_{1:T}|a_{1:T})$ and $p(o^n_{1:T}|a_{1:T})$, let $\mathcal{Q}$ denote the solution set of $\mathcal{L}_{\mathrm{KL}}$:
\begin{equation*}
    \mathcal{Q} \coloneqq \mathop{\arg\min}\limits_{q(o_t|o^n_t)} \min\limits_{_{q(o^n_t|o_t)}}\mathcal{L}_{\mathrm{KL}}.
\end{equation*}
It follows that $\mathcal{Q}$ equals the set of posterior denoising distributions of noise functions in $\mathcal{H}^p_{f_n}$:
\begin{equation}
    \mathcal{Q} = \left \{ p(o_t |o^{n_i}_t)|f_{n_i}\in \mathcal{H}^p_{f_n}\right \}.
\end{equation}
\end{theorem}

% \begin{theorem}
% \label{theorem_homo}
% Given $p(o_{1:T}|a_{1:T})$ and $p(o^n_{1:T}|a_{1:T})$, the optimal denoising model $m^*_{\mathrm{de}}$, which satisfies $\delta(o_t-m^*_{\mathrm{de}}(o^n_t)) = q^*(o_t|o^n_t) = \arg\min \mathcal{L}_{\mathrm{KL}}$, is homogeneous to $f_n$, i.e., 
% \begin{equation}
%     m^*_{\mathrm{de}} \equiv_p f_{n}.
% \end{equation}
% \end{theorem}

% \zxn{details?, when they are equal}
Generally speaking, since homogeneous noise functions are theoretically indistinguishable in the unsupervised setting, we can only assure that $m_{\mathrm{de}}$ learns to transfer cluttered observations back to clean ones according to a noise function in $\mathcal{H}^p_{f_n}$. In Appendix~\ref{appendix_subsec_optimality}, we further reveal the relationship between the number of homogeneous noise functions and properties of $p(o_{1:T}|a_{1:T})$. We also discuss possible ways to reduce the number of homogeneous noise functions in Sec.~\ref{subsec_homo} so that $\mathcal{Q}$ only contains $p(o_t|o^n_t)$.

% Unfortunately, $\mathcal{L}_{\mathrm{KL}}$ is difficult to optimize as it is non-trivial to estimate $\mathbb{E}_{p(o^n_{1:T}|a_{1:T})}[q(o_{1:T}|o^n_{1:T})]$. To address this problem, we additionally introduce a learnable noisy distribution $q(o^n_t|o_t)$ and extend $\mathcal{L}_{\mathrm{KL}}$ to $\mathcal{L}^{\mathrm{joint}}_{\mathrm{KL}}$, which is the KL-divergence between action-conditioned joint probabilities:
% \begin{equation}
% \begin{aligned}
%     \mathcal{L}^{\mathrm{joint}}_{\mathrm{KL}} &= \mathrm{D}_{\mathrm{KL}} \Big(p(o^n_{1:T}|a_{1:T})q(o_{1:T}|o^n_{1:T}) \\
%     &\quad\quad\quad\quad\quad\quad \big\Vert p(o_{1:T}|a_{1:T})q(o^n_{1:T}|o_{1:T})\Big).
% \end{aligned}
% \end{equation}
% With $q(o^n_t|o_t)$ being optimizable, we demonstrate that $\mathcal{L}^{\mathrm{joint}}_{\mathrm{KL}}$ is equivalent to $\mathcal{L}_{\mathrm{KL}}$ in the sense that the optimal $q^*(o_t|o^n_t)$ is identical for both objectives (proof in Appendix X):
% \begin{equation}
%     \mathop{\arg\min}_{q(o_t|o^n_t)}\; \min_{q(o^n_t|o_t)}\mathcal{L}^{\mathrm{joint}}_{\mathrm{KL}}= \mathop{\arg\min}_{q(o_t|o^n_t)}\; \mathcal{L}_{\mathrm{KL}}.
% \end{equation}

To simplify the computation, we show in Appendix~\ref{appendix_subsec_distribution} that $\mathcal{L}_{\mathrm{KL}}$ leads to the following objective, where $C$ is a constant:
\begin{equation}
\begin{aligned}
\label{eq_distribution_matching}
    \mathcal{L}_{\mathrm{KL}} &= \mathbb{E}_{p(o^n_{1:T}|a_{1:T})} \Big[\mathrm{D}_{\mathrm{KL}}\big (q(o_{1:T}|o^n_{1:T})\|p(o_{1:T}|a_{1:T})\big ) \\
    &\quad\quad\quad\quad\quad\quad -\mathbb{E}_{q(o_{1:T}|o^n_{1:T})}[\log q(o^n_{1:T}|o_{1:T})]\Big ] + C\\
    &= \mathbb{E}_{p(o^n_{1:T}|a_{1:T})}\mathbb{E}_{q(o_{1:T}|o^n_{1:T})} \Big [ -\log p(o_{1:T}|a_{1:T}) \\
    &\quad\quad\quad\quad\quad\quad-\log q(o^n_{1:T}|o_{1:T}) \Big] +C.
\end{aligned}
\end{equation}
% where $p(m_{\mathrm{de}}(o^n_{1:T})|a_{1:T})$ stands for $p(\{m_{\mathrm{de}}( o^n_{i}) \}_{i=1}^{T}|a_{1:T})$, and the equation $(*)$ holds as $q(o_t|o^n_t)=\delta(o_t-m_{\mathrm{de}}(o^n_t))$. 

Intuitively, the first term regularizes the transferred observations to follow the clean environments' latent dynamics $p(o_{1:T}|a_{1:T})$. The second term ensures that the transferred observations remain relevant to the cluttered observations and thus preserve necessary information.

\subsection{Adaptation with Pre-trained World Models}
\label{subsec_adaptation}

% Estimating $p(o_{1:T}|a_{1:T})$ with World Models.
% Due to the high dimensionality of the observation space, RSSM~\citep{} is proposed to estimate $p(o_{1:T}|a_{1:T})$ and facilitate trajectory generation, namely the \textit{world models}.

Based on the above analyses, we now present the Self-Consistent Model-based Adaptation (SCMA) method, a practical adaptation algorithm that mitigates distractions by optimizing the denoising model with Eq.~\ref{eq_distribution_matching}.

Specifically, Eq.~\ref{eq_distribution_matching} involves calculating the action-conditioned distribution $p(o_{1:T}|a_{1:T})$, which we estimate with a pre-trained world model~\citep{hafner2019learning,hafner2023mastering}. Given a clean trajectory $\tau = \left \{ o_1, a_1, \cdots, o_T, a_T\right \}$, the world model estimates $\log p(o_{1:T}|a_{1:T})$ with $\log p_{\scriptscriptstyle \mathrm{wm}}(o_{1:T}|a_{1:T})$ by maximizing the following evidence lower bound (ELBO):
\begin{equation}
\label{eq_rssm}
        \begin{aligned} 
            & \log p_{\scriptscriptstyle \mathrm{wm}}(o_{1:T}|a_{1:T}) = \log \int p_{\scriptscriptstyle \mathrm{wm}}(o_{1:T}, s_{1:T}|a_{1:T})\mathrm{d} s_{1:T} \\
            \geq &\sum\limits_{t=1}^T \mathbb{E}_{q_{\scriptscriptstyle \mathrm{wm}}(s_{1:T}|a_{1:T}, o_{1:T})}[\underbrace{\log p_{\scriptscriptstyle \mathrm{wm}}(o_t|s_{\leq t}, a_{<t})}_{\mathcal{J}_{o}^t} \\
            &\quad -\underbrace{\mathrm{D}_\mathrm{KL}(q_{\scriptscriptstyle \mathrm{wm}}(s_t|s_{<t},a_{<t}, o_t)\Vert p_{\scriptscriptstyle \mathrm{wm}}(s_t|s_{<t}, a_{<t})}_{\mathcal{J}_{kl}^t})].
    \end{aligned}
\end{equation}

In the above objective, the KL-divergence objective $\mathcal{J}_{kl}^t$ enables the model's generation ability by minimizing the distance between the prior and posterior distribution. The reconstruction objective $\mathcal{J}^t_{o}$ enforces the model to capture the visual essence of the task by predicting the subsequent observations, which facilitates the later adaptation.

\paragraph{Self-consistent Model-based Adaptation} Before adaptation, we first pre-train the policy and world model in clean environments. Then we deploy the pre-trained policy and our denoising model into the distracting environment to collect trajectory $\left \{o^n_1, a_1, \cdots, o^n_T, a_T \right \}$. By estimating $p(o_{1:T}|a_{1:T})$ with the pre-trained world model, optimizing Eq.~\ref{eq_distribution_matching} leads to the following self-consistent reconstruction loss $\mathcal{L}^t_{sc}$ and noisy reconstruction loss $\mathcal{L}^t_{n}$. It should be noted that the world model is frozen during adaptation. We choose to drop a similar KL-loss term as in Eq.~\ref{eq_rssm} because we empirically find it to have a negative impact on adaptation by harming the reconstruction results, consistent with previous works~\citep{higgins2017beta,chen2018isolating}. The detailed derivation is provided in Appendix~\ref{appendix_subsec_scma}.
\begin{equation*}
\resizebox{\linewidth}{!}{$
    \displaystyle{
        \begin{aligned} 
        \mathcal{L}^t_{sc} &= -\mathbb{E}_{q(o_{1:T}|o^n_{1:T})}\mathbb{E}_{q_{\scriptscriptstyle \mathrm{wm}}(s_{1:T}|a_{1:T}, o_{1:T})}[\log p_{\scriptscriptstyle \mathrm{wm}}(o_t|s_{\leq t},a_{<t})], \\
        \mathcal{L}^t_{n} &= -\mathbb{E}_{q(o_{1:T}|o^n_{1:T})}[\log q(o^n_{t}|o_t)].
    \end{aligned}
    }
$}  
\end{equation*}

$\mathcal{L}^t_{sc}$ encourages the denoising model to transfer cluttered observations to clean ones so that the transferred observations will conform to the prediction of the world model. $\mathcal{L}^t_{n}$ 
prevents the denoising model from ignoring the cluttered observations and thus outputting clean yet irrelevant observations. In practice, we implement $q(o^n_t|o_t)=\delta(o^n_t - m_{\mathrm{n}}(o_t))$ with a noisy model $m_{\mathrm{n}}$, and $q(o_{1:T}|o^n_{1:T})=\prod_t q(o_t|o^n_t)=\prod_t \delta(o_t - m_{\mathrm{de}}(o^n_t))$ with a denoising model $m_{\mathrm{de}}$.

\subsection{Boosting Adaptation by Reducing Homogeneous Noise Functions}
\label{subsec_homo}
% 1. p(o|a)
% 2. prior finding similarity between o and o'
% 3. reward

As discussed in Theorem~\ref{theorem_homo}, the solution set of $\mathcal{L}_{\mathrm{KL}}$ equals the set of posterior denoising distributions of noise functions in $\mathcal{H}^p_{f_n}$. To promote the adaptation, we propose two practical techniques to help the denoising distribution $q(o_t|o^n_t)$ converge to the target posterior denoising distribution $p(o_t|o^n_t)$ by reducing the number of homogeneous noise functions.

% To help the denoising distribution $q(o_t|o^n_t)$ converge to the target posterior denoising distribution $p(o_t|o^n_t)$, we propose two practical techniques: reducing the number of homogeneous noise functions or limiting the hypothesis set.

\paragraph{Leverage Rewards} If reward signals are available in distracting environments, they can naturally boost adaptation by reducing the number of homogeneous noise functions. Loosely speaking, noise functions with the same $p(o^{n}_{1:T}|a_{1:T})$ but different $p(o^{n}_{1:T},r_{1:T}|a_{1:T})$ are no longer homogeneous if rewards are available. A detailed explanation is provided in Appendix~\ref{appendix_subsec_optimality}. The derivation in Sec.~\ref{subsec_mitigating} can be simply extended to include rewards by redefining $\mathcal{L}_{\mathrm{KL}}$ as below (details in Appendix~\ref{appendix_subsec_scma}):
\begin{equation*}
\begin{aligned}
    \mathcal{L}_{\mathrm{KL}} \coloneqq & \mathrm{D}_{\mathrm{KL}} \Big(p(o^n_{1:T},r_{1:T}|a_{1:T})q(o_{1:T}|o^n_{1:T}) \\
    &\quad\quad\quad\quad \big\Vert p(o_{1:T},r_{1:T}|a_{1:T})q(o^n_{1:T}|o_{1:T})\Big),
\end{aligned}
\end{equation*}
which leads to the reward prediction loss:
\begin{equation*}
    \resizebox{\linewidth}{!}{$
        \displaystyle{
            \mathcal{L}^t_{rew} = -\mathbb{E}_{q(o_{1:T}|o^n_{1:T})}\mathbb{E}_{q_{\scriptscriptstyle \mathrm{wm}}(s_{1:T}|a_{1:T}, o_{1:T})}[\log p_{\scriptscriptstyle \mathrm{wm}}(r_t|s_{\leq t},a_{<t})].
        }
    $}
\end{equation*}

$\mathcal{L}^t_{rew}$ encourages the transferred observations to contain sufficient information of rewards and ignore reward-irrelevant distractions. The final adaptation loss of SCMA is:
\begin{equation}
\label{eq_scma_loss}
    \mathcal{L}^t_{\mathrm{SCMA}}= \mathcal{L}^t_{sc} + \mathcal{L}^t_{n} + \mathcal{L}^t_{rew}.
\end{equation}

\paragraph{Limit the Hypothesis Set of the Denoising Model} For specific types of distractions, we can further encode some inductive bias in the denoising model architecture. Therefore, we can prevent $q(o_t|o^n_t)$ from converging to the posterior denoising distributions of certain homogeneous noise functions by limiting the hypothesis set. For example, we can implement the denoising model as a mask model $m_{\mathrm{mask}}: \mathbb{R}^{h\times w\times c}\mapsto[0,1]^{h\times w\times c}$ to handle background distractions. However, to verify the generality of SCMA, we refrain from assuming the type of distractions and implement the denoising model as a generic image-to-image network by default. Detailed implementations are provided in Appendix~\ref{appendix_subsec_implementation}.

In summary, we propose an adaptation framework with a two-stage pipeline: 1) pre-training the policy and world model in clean environments to master skills and capture the environments' latent dynamics $p(o_{1:T}|a_{1:T})$. 2) adapting the policy to visually distracting environments by optimizing $q(o_t|o^n_t)$ with Eq.~\ref{eq_scma_loss} to transfer cluttered trajectories to clean ones. The pipeline is illustrated with Fig.~\ref{fig_scma_algo} along with pseudocode in Appendix~\ref{appendix_subsec_code}.

% By reducing the hypothesis set, we empirically show that SCMA can achieve appealing adaptation results towards specific types of distraction. 

% we refrain from assuming specific types of distractions and implement the denoising model as a generic image-to-image network for most experiments unless specified otherwise.

\begin{table*}[tb]
\subcaptionbox{$\mathrm{video\_hard}$\label{table_video_hard}}[\linewidth]{
    \tablestyle{3.5pt}{1.1}
    \centering
    % \scriptsize
    \footnotesize
    \begin{tabular}{c|cc|cc|ccc|ccc}
        \toprule
        % \multirow{2}{*}{$\mathrm{video\_hard}$} & \multirow{2}{*}{SCMA} & SCMA \scriptsize{(w/o r)} & \multirow{2}{*}{MoVie} & \multirow{2}{*}{PAD}  & \multirow{2}{*}{SVEA} & \multirow{2}{*}{Dr. G} & \multirow{2}{*}{SGQN} & \multirow{2}{*}{TIA} & \multirow{2}{*}{TPC} & \multirow{2}{*}{DreamerPro} \\
        % & & (w/o r) & & & & & & & & \\
        $\mathrm{video\_hard}$ & SCMA & SCMA \scriptsize{(w/o r)} & MoVie & PAD & SVEA & Dr. G & SGQN & TIA & TPC & DreamerPro \\
        
        \midrule

        \text{ball\_in\_cup-catch} &  \textbf{809\scriptsize{$\pm$114}} & 215\scriptsize{$\pm$60} & 41\scriptsize{$\pm$20} & 130\scriptsize{$\pm$47} & 498\scriptsize{$\pm$147} & 635\scriptsize{$\pm$26} & 782\scriptsize{$\pm$57} & 329\scriptsize{$\pm$466} & 220\scriptsize{$\pm$207} & 378\scriptsize{$\pm$231} \\
        
        \text{cartpole-swingup}  &  \textbf{773\scriptsize{$\pm$51}} & 145\scriptsize{$\pm$40} & 83\scriptsize{$\pm$2} & 123\scriptsize{$\pm$21} & 401\scriptsize{$\pm$38} & 545\scriptsize{$\pm$23} & 544\scriptsize{$\pm$43} & 98\scriptsize{$\pm$22} & 219\scriptsize{$\pm$19} & 365\scriptsize{$\pm$48} \\
        
        \text{finger-spin} & \textbf{948$\pm$5} & 769\scriptsize{$\pm$182} & 2\scriptsize{$\pm$0} & 96\scriptsize{$\pm$11} & 307\scriptsize{$\pm$24} & - & 822\scriptsize{$\pm$24} & 146\scriptsize{$\pm$93} & 315\scriptsize{$\pm$40} & 427\scriptsize{$\pm$299}\\

        \text{walker-stand} & \textbf{953$\pm$4} & 328\scriptsize{$\pm$30} & 127\scriptsize{$\pm$23} & 336\scriptsize{$\pm$22} & 747\scriptsize{$\pm$43} & - & 851\scriptsize{$\pm$24} & 117\scriptsize{$\pm$9} & 840\scriptsize{$\pm$98} & 941\scriptsize{$\pm$14} \\

        \text{walker-walk} & 722\scriptsize{$\pm$89}  & 129\scriptsize{$\pm$19} & 39\scriptsize{$\pm$13} & 108\scriptsize{$\pm$33} & 385\scriptsize{$\pm$63} & \textbf{782\scriptsize{$\pm$37}} & 739\scriptsize{$\pm$21} & 84\scriptsize{$\pm$55} & 402\scriptsize{$\pm$57} & 617\scriptsize{$\pm$159} \\
        \bottomrule
    \end{tabular}
}
\\
\vspace{0.2em}
\\
\subcaptionbox{$\mathrm{moving\_view}$\label{table_moving}}[0.4\linewidth]{
    \centering
    \tablestyle{3pt}{1.1}
    % \scriptsize
    \footnotesize
    \begin{tabular}{c|cccc}
    \toprule
    $\mathrm{moving\_view}$ &  SCMA & MoVie & PAD & SGQN \\
    \midrule
    \text{ball\_in\_cup-catch} &  745\scriptsize{$\pm$121} & \textbf{951\scriptsize{$\pm$10}} & 750\scriptsize{$\pm$32} & 857\scriptsize{$\pm$64} \\
    
    \text{cartpole-swingup}  &  708\scriptsize{$\pm$76} & 167\scriptsize{$\pm$25} & 561\scriptsize{$\pm$86} & \textbf{788\scriptsize{$\pm$65}} \\
    
    \text{finger-spin} &  \textbf{952\scriptsize{$\pm$10}} & 896\scriptsize{$\pm$21} & 603\scriptsize{$\pm$28} & 702\scriptsize{$\pm$56} \\
    
    \text{walker-stand} & \textbf{977\scriptsize{$\pm$16}} & 712\scriptsize{$\pm$11} & 955\scriptsize{$\pm$15} & 961\scriptsize{$\pm$2} \\
    
    \text{walker-walk} & \textbf{922\scriptsize{$\pm$55}} & 810\scriptsize{$\pm$7} & 645\scriptsize{$\pm$21} & 769\scriptsize{$\pm$36} \\ 
    \bottomrule
    \end{tabular}
}
% \hspace{-1em}
\subcaptionbox{$\mathrm{color\_hard}$\label{table_color_hard}}[0.6\linewidth]{
    \centering
    \tablestyle{3pt}{1.1} 
    % \scriptsize
    \footnotesize
    \begin{tabular}{c|ccccc}
    \toprule
    \text{$\mathrm{color\_hard}$} & SCMA & MoVie & PAD & SGQN & SVEA \\
    \midrule
    \text{ball\_in\_cup-catch} & 817\scriptsize{$\pm$64} & 67\scriptsize{$\pm$41} & 563\scriptsize{$\pm$50} & 881\scriptsize{$\pm$61} & \textbf{961\scriptsize{$\pm$7}} \\
    
    \text{cartpole-swingup}  & 809\scriptsize{$\pm$15} & 102\scriptsize{$\pm$14} & 630\scriptsize{$\pm$63} & 773\scriptsize{$\pm$80} & \textbf{837\scriptsize{$\pm$23}} \\
    
    \text{finger-spin} & 965\scriptsize{$\pm$2} & 652\scriptsize{$\pm$10}  & 803\scriptsize{$\pm$72} & 847\scriptsize{$\pm$80} & \textbf{977\scriptsize{$\pm$5}} \\
    
    \text{walker-stand} & \textbf{984\scriptsize{$\pm$11}} & 121\scriptsize{$\pm$14} & 797\scriptsize{$\pm$46} & 867\scriptsize{$\pm$81} & 942\scriptsize{$\pm$26} \\
    
    \text{walker-walk} & \textbf{954\scriptsize{$\pm$7}} & 38\scriptsize{$\pm$3}  &468\scriptsize{$\pm$74} & 828\scriptsize{$\pm$84} & 760\scriptsize{$\pm$145} \\
    % \midrule
    % \text{Average} & 856 & 214 & 652 & 850 & 895 \\
    \bottomrule
    \end{tabular}
}
\\
\vspace{0.1em}
\\
\subcaptionbox{$\mathrm{occlusion}$\label{table_occlusion}}[0.4\linewidth]{
    \centering
    \tablestyle{3pt}{1.1}
    % \scriptsize
    \footnotesize
    \begin{tabular}{c|cccc}
    \toprule
    \text{$\mathrm{occlusion}$} & SCMA & MoVie & PAD & SGQN \\
    \midrule
    \text{ball\_in\_cup-catch} & \textbf{899\scriptsize{$\pm$41}} & 33\scriptsize{$\pm$18} & 145\scriptsize{$\pm$6} & 642\scriptsize{$\pm$74} \\
    
    \text{cartpole-swingup}  & \textbf{779\scriptsize{$\pm$10}} & 120\scriptsize{$\pm$32} & 142\scriptsize{$\pm$9} & 127\scriptsize{$\pm$18} \\
    
    \text{finger-spin} & \textbf{920\scriptsize{$\pm$1}} & 1\scriptsize{$\pm$0} & 15\scriptsize{$\pm$9} & 117\scriptsize{$\pm$22} \\
    
    \text{walker-stand} & \textbf{976\scriptsize{$\pm$17}} & 124\scriptsize{$\pm$21} & 305\scriptsize{$\pm$16} & 376\scriptsize{$\pm$87} \\
    
    \text{walker-walk} & \textbf{902\scriptsize{$\pm$51}} & 52\scriptsize{$\pm$15} & 94\scriptsize{$\pm$24} & 118\scriptsize{$\pm$34} \\
    % \midrule
    % \text{Average} & 856 & 214 & 652 & 850 & 895 \\
    \bottomrule
    \end{tabular}
}
% \hspace{0.1em}
\subcaptionbox{Table-top Manipulation tasks in RL-ViGen.\label{table_rlvigen}}[0.6\linewidth]{
    \centering
    \tablestyle{2.5pt}{1.0}
    % \scriptsize
    \footnotesize
    \begin{tabular}{c|ccccc}
    \toprule
    \text{RL-ViGen} & SCMA & SGQN & SRM & SVEA & CURL \\
    \midrule
    Door (easy) & \textbf{416\scriptsize{$\pm$26}} & 391\scriptsize{$\pm$95}  &  337\scriptsize{$\pm$110} & 268\scriptsize{$\pm$136} & 6\scriptsize{$\pm$5} \\
    Door (extreme)  & \textbf{380\scriptsize{$\pm$30}} & 160\scriptsize{$\pm$122} & 31\scriptsize{$\pm$18}   & 62\scriptsize{$\pm$56} & 2\scriptsize{$\pm$1} \\
    \midrule
    Lift (easy) &  19\scriptsize{$\pm$5}   & 31\scriptsize{$\pm$17} & \textbf{69\scriptsize{$\pm$32}}   & 43\scriptsize{$\pm$18} & 0\scriptsize{$\pm$0} \\ 
    Lift (extreme) & \textbf{15\scriptsize{$\pm$9}}   & 7\scriptsize{$\pm$7} & 0\scriptsize{$\pm$0} & 8\scriptsize{$\pm$5} & 0\scriptsize{$\pm$0} \\
    \midrule
    TwoArm (easy) & 340\scriptsize{$\pm$27} & 349\scriptsize{$\pm$23} & \textbf{419\scriptsize{$\pm$45}}  & 414\scriptsize{$\pm$58} & 150\scriptsize{$\pm$20} \\
    TwoArm (extreme) & 227\scriptsize{$\pm$24} & \textbf{257\scriptsize{$\pm$31}} & 161\scriptsize{$\pm$27} & 155\scriptsize{$\pm$18}  & 147\scriptsize{$\pm$15} \\
    \bottomrule
    \end{tabular}
}
\vspace{-0.5em}
\caption{Performance (mean $\pm$ std) in visually distracting environments. We report the performance of SCMA and baseline methods in DMControl and RL-ViGen across various distracting settings. The best algorithm is \textbf{bold} for every task.}
\label{table_dmc_all}
\vspace{-1em}
\end{table*}

\begin{figure*}[t]
    \centering
    \centerline{\includegraphics[width=0.8\linewidth]{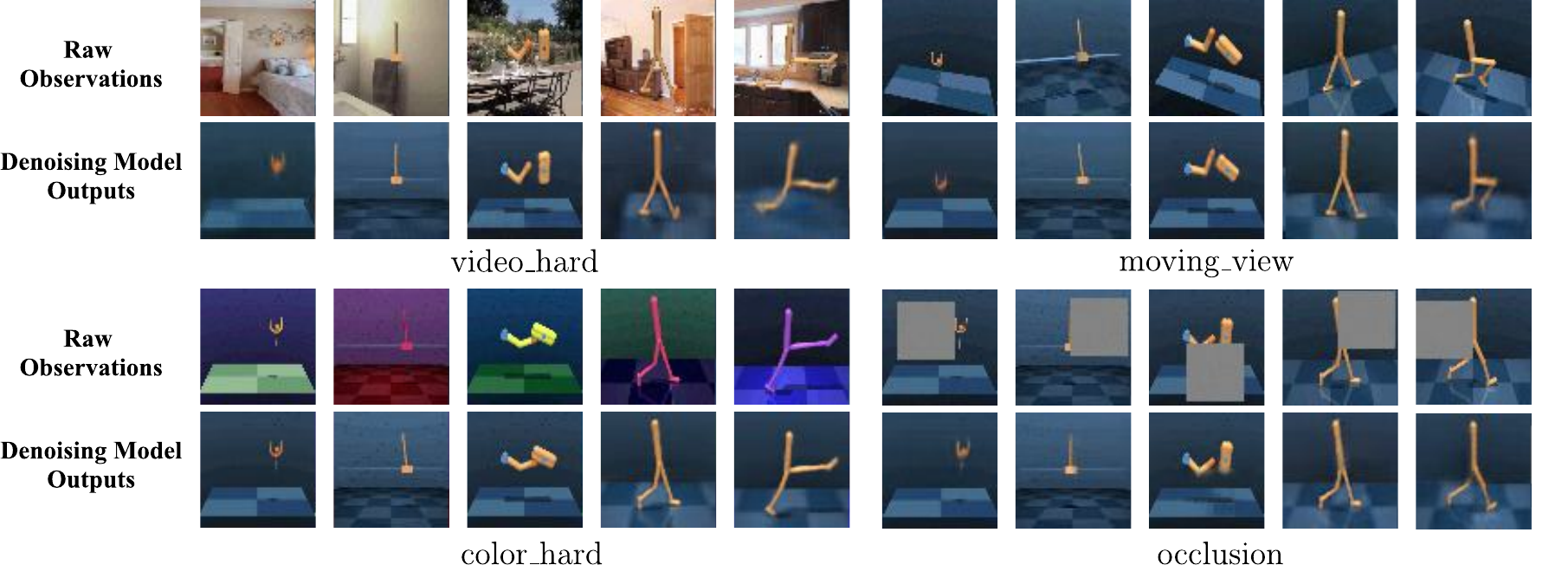}}
    \vspace{-0.5em}
    \caption{Visualization of the raw observations and the denoising model's outputs in various distracting environments.}
    \label{fig_denoising_visualization}
    \vspace{-1em}
\end{figure*}

\section{Experiment}
\label{sec_exp}

In this section, we evaluate the capability of SCMA by addressing the following questions:
% \fy{Maybe we can abbreviate these as Q1, Q2 and Q3 and put references in later results, which makes it clear that which question is answered by which part of experiments}
\begin{itemize}
    \item Can SCMA fill the performance gap caused by various types of distractions?
    \item Can SCMA generalize across various tasks or policies from different algorithms?
    \item How does each loss component contribute to the results? Can SCMA still handle distractions without rewards?
    \item Can SCMA converge faster compared to other adaptation-based methods or directly training from scratch in visually distracting environments?
    % \ycy{
    % \item Can SCMA fill the generalization gap by directly adapting the agent to visually distracting environments and perform better than augmentation-based methods?
    % \item Does each component in SCMA contribute to generalization? 
    % % \item When deployed into environments with no rewards, is it possible for SCMA to boost performance with visual signals alone?
    % \item How much does SCMA reduce the downstream timesteps compared to directly training from scratch in visually distracting environments?
    % }
\end{itemize}

% \begin{table}[tb]
%     \centering
%     \begin{tabular}{c|cccc}
%         \toprule
%          Setting & SCMA & MoVie & PAD & SGQN \\
%          \midrule
%          $\mathrm{video\_hard}$ & \textbf{841.0} & 58.4 & 158.6 & 747.6 \\
%          $\mathrm{moving\_view}$ & \textbf{860.8} & 707.2 & 702.8 & 815.4 \\
%          $\mathrm{color\_hard}$ & \textbf{883.2} & 196.0 & 652.2 & 839.2 \\
%          $\mathrm{occlusion}$ & \textbf{895.2} & 66.0 & 140.2 & 276.0 \\
%          \midrule
%          Averaged & \textbf{870.1} & 258.4 & 413.5 & 669.6 \\
%          \bottomrule
%     \end{tabular}
%     \caption{Caption}
%     \label{table_dmc_average}
% \end{table}

\subsection{Experiment Setup}
\paragraph{Environments} To measure the effectiveness of SCMA, we follow the settings from the commonly adopted DMControlGB~\citep{hansen2021generalization,hansen2021stabilizing,Bertoin2022LookWY}, DMControlView~\citep{yang2024movie} and RL-ViGen~\citep{yuan2024rl,chen2024focus}. The agent is asked to perform continuous control tasks in visually distracting environments, including video distracting background ($\mathrm{video\_hard}$), moving camera views ($\mathrm{moving\_view}$), and randomized colors ($\mathrm{color\_hard}$). We also evaluate the agent's performance in a more challenging $\mathrm{occlusion}$ setting by randomly masking $1/4$ of each observation. We provide a visualization of every distracting environment in Fig.~\ref{fig_overview_dmc} in the Appendix. Unless otherwise stated, the result of each task is evaluated over $3$ seeds and we report the average performance of the policy in the last episode.

\paragraph{Baselines} We compare SCMA to the state-of-the-art adaptation-based baselines: PAD~\citep{Hansen2020SelfSupervisedPA}, MoVie~\citep{yang2024movie}. We also include comparison with other kinds of methods, including augmentation-based methods: SVEA~\citep{hansen2021stabilizing}, SGQN~\citep{Bertoin2022LookWY}, Dr. G~\citep{ha2023dream}; and task-induced methods: TIA~\citep{fu2021learning}, TPC~\citep{nguyen2021temporal}, DreamerPro~\citep{deng2022dreamerpro}. Following the official design~\citep{hansen2021generalization}, the augmentation-based methods use random overlay with images from Place365~\citep{zhou2017places}. 
Task-induced methods directly learn the structured representations in distracting environments. Adaptation-based methods will first be pre-trained in the clean environments for $1$M timesteps and then adapt to the distracting environments for $0.1$M timesteps ($0.4$M for $\mathrm{video\_hard}$ and $0.5$M for RL-ViGen). By default, SCMA adapts a pre-trained Dreamer policy~\citep{hafner2019dream} to distracting environments. More details can be found in Appendix~\ref{appendix_baseline}.

\subsection{Adaptability to Visual Distractions}
% with natural video background and varying lighting conditions
We first evaluate the adaptation ability of SCMA by measuring its performance in the challenging visual generalization benchmarks. Before adapting to the visually distracting environments, we first pre-train the policy and world model in the clean training environment (see Fig.~\ref{fig_env_raw} in the Appendix). Then we adapt the agent to visually distracting environments leveraging the pre-trained world model. The experiment results in Table~\ref{table_dmc_all} show that SCMA significantly reduces the performance gap caused by distractions and achieves appealing performance compared to augmentation-based methods. While remaining competitive in the $\mathrm{color\_hard}$ setting, SCMA outperforms the best baseline method in most tasks in other $3$ settings. Moreover, SCMA obtains the best performance in all tasks in the $\mathrm{occlusion}$ setting, which is a common scenario for real-world robot controls. To verify the idea of boosting adaptation by reducing the hypothesis set (Sec.~\ref{subsec_homo}), we implement the denoising model with specific architectures and conduct experiments in the table-top manipulation tasks with distracting settings from RL-ViGen. Further details are included in Appendix~\ref{appendix_subsec_implementation}. Following previous work \citep{yuan2024rl}, we report the scores under the $\mathrm{eval\_easy}$ and $\mathrm{eval\_extreme}$ settings in Table~\ref{table_rlvigen}. The results show that SCMA achieves the best performance in half of the scenarios and remains comparable to other methods in the remaining ones. We believe one way to further improve the performance is to incorporate stronger world models~\citep{ding2024diffusionworldmodelfuture}, which we leave to future works.

We visualize how $m_\mathrm{de}$ transfers cluttered observations to clean ones in Fig.~\ref{fig_denoising_visualization}, where it effectively mitigates various types of distractions and restores the task-relevant objects correctly. The qualitative results also indicate that our method can effectively handle distractions not only for large embodiments like $\mathrm{walker}\text{-}\mathrm{walk}$, but also for challenging small embodiments such as $\mathrm{ball\_in\_cup}\text{-}\mathrm{catch}$ and $\mathrm{cartpole}\text{-}\mathrm{swingup}$, which task-induced methods often fail to manage. 

\begin{table}[tb]
    \centering
    \footnotesize
    \begin{tabular}{c|cc}
        \toprule
        $\mathrm{occlusion}$ & SGQN & SGQN\scriptsize{$+$SCMA} \\
        \midrule
        $\mathrm{ball\_in\_cup\text{-}catch}$ & 642\scriptsize{$\pm$74} & 775\scriptsize{$\pm$151} \\
        $\mathrm{cartpole\text{-}swingup}$ & 127\scriptsize{$\pm$18} & 337\scriptsize{$\pm$51} \\
        $\mathrm{finger\text{-}spin}$ & 117\scriptsize{$\pm$22} & 133\scriptsize{$\pm$19} \\
        $\mathrm{walker\text{-}stand}$ & 376\scriptsize{$\pm$87} & 884\scriptsize{$\pm$63} \\
        $\mathrm{walker\text{-}walk}$ & 118\scriptsize{$\pm$34} &  465\scriptsize{$\pm$101} \\
        \midrule
        Averaged & 276.0 & 518.8 \scriptsize{(88.0\%$\uparrow$)}\\
        \bottomrule
    \end{tabular}
    \vspace{-0.5em}
    \caption{Performance (mean $\pm$ std) in $\mathrm{occlusion}$ environment. The results show that the denoising model can boost SGQN's performance in a plug-and-play manner.}
    \vspace{-1.0em}
    \label{table_sgqn_scma}
\end{table}

\subsection{Versatility of the Denoising Model}

We conduct experiments to measure the versatility of the denoising model from two aspects: 1) can the denoising model generalize across tasks with the same robot? 2) is the denoising model applicable to policies from different algorithms? To answer the above questions, we first cross-evaluate the capability of the denoising model between $\mathrm{walker\text{-}walk}$ and $\mathrm{walker\text{-}stand}$ in the $\mathrm{video\_hard}$ environment. Specifically, we take the denoising model adapted to one task and directly evaluate its performance in another task. The results in Table~\ref{table_cross_task} in the Appendix indicate that the achieved denoising model is not restricted to a specific task and exhibits appealing zero-shot generalization capability. To verify that the denoising model is agnostic to policies, we first optimize the denoising model with trajectories collected by a Dreamer policy. Then we combine the obtained denoising model with an SGQN policy in a plug-and-play manner and measure the performance in the $\mathrm{occlusion}$ setting. While SGQN reaches appealing results in other settings, it performs poorly under occlusions. However, Table~\ref{table_sgqn_scma} demonstrates that incorporating the denoising model can improve the performance of SGQN by $88\%$. Therefore, SCMA can serve as a convenient component to promote performance under certain distractions without modifying the policy. However, there is a disparity between the performance of SGQN policy with SCMA and Dreamer policy with SCMA, which we attribute to the policy encoder. Since the encoder of Dreamer policy leverages the long-term representation extracted by the world model, it is less susceptible to small mistakes made by the denoising model.

\subsection{Adaptation Without Rewards}

While SCMA utilizes both visual and reward signals for the best adaptation results, the ability to adapt without rewards is also important. To address this issue, we conduct experiments in the $\mathrm{video\_hard}$ environments to investigate how different loss components affect the final adaptation results.

%  we provide the ablation results in Fig.~\ref{ablation_different_loss} and visualization in Fig.~\ref{ablation_module_masks}.
To better understand the impact of different losses, we separately removed the $3$ loss components from SCMA during adaptation, namely self-consistent reconstruction loss $\mathcal{L}^t_{sc}$, reward prediction loss $\mathcal{L}^t_{rew}$, and noisy reconstruction loss $\mathcal{L}^t_{n}$. From the ablation results in Fig.~\ref{fig_different_loss} in the Appendix, we can see that removing the self-consistent reconstruction loss leads to the most significant decrease, indicating that proposed $\mathcal{L}^t_{sc}$ plays a crucial role in adaptation. Another finding is that the reward loss can promote better adaptation by encouraging the denoising model to focus on some miniature yet critical features, such as the ball in $\mathrm{ball\_in\_cup}\text{-}\mathrm{catch}$ and the pole in $\mathrm{cartpole}\text{-}\mathrm{balance}$. While $\mathcal{L}^t_{rew}$ contributes considerably to the final adaptation results, Fig.~\ref{fig_average_dmc} in the Appendix demonstrates that SCMA without rewards still achieves the highest average performance among all adaptation-based methods. The noisy reconstruction loss mainly aims to preserve the connection between the cluttered and transferred observations. Intuitively, removing $\mathcal{L}^t_{n}$ from Eq.~\ref{eq_distribution_matching} will cause a \textit{mode-seeking} problem~\citep{cheng1995mean}, where the denoising model will prefer the mode of $\log p(o_{1:T}|a_{1:T})$ and thus transfer cluttered observations to clean yet irrelevant observations.

\begin{figure}[t]
\centering
\begin{subfigure}[b]{\linewidth}
    \centering
    \includegraphics[width=\linewidth]{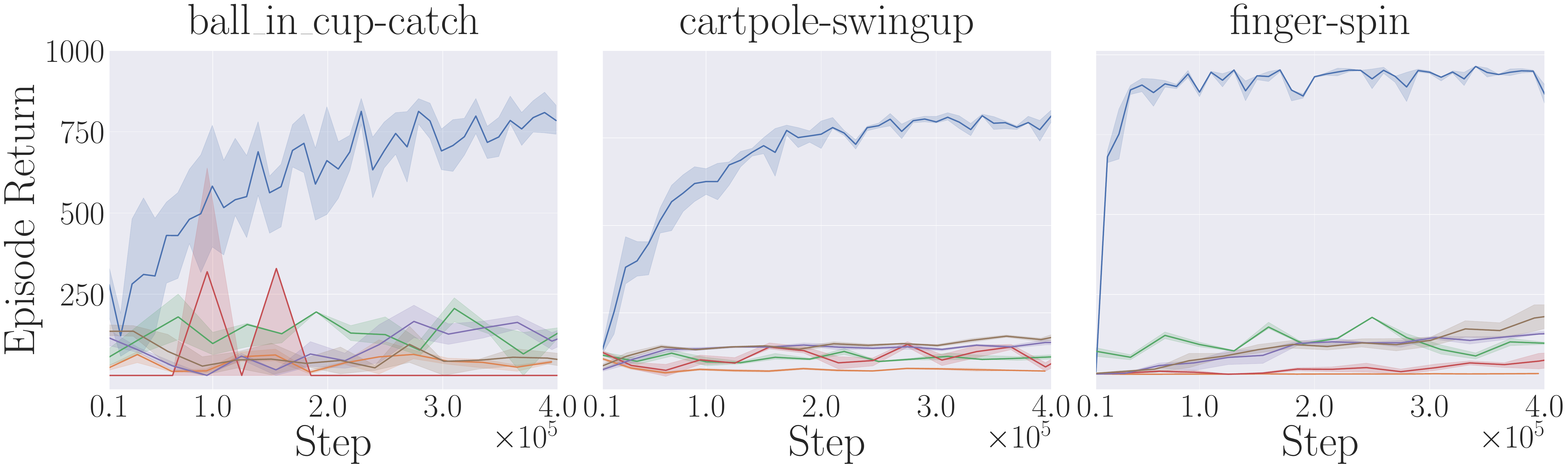}
\end{subfigure}
\begin{subfigure}[b]{0.7\linewidth}
    \centering
    \includegraphics[width=\linewidth]{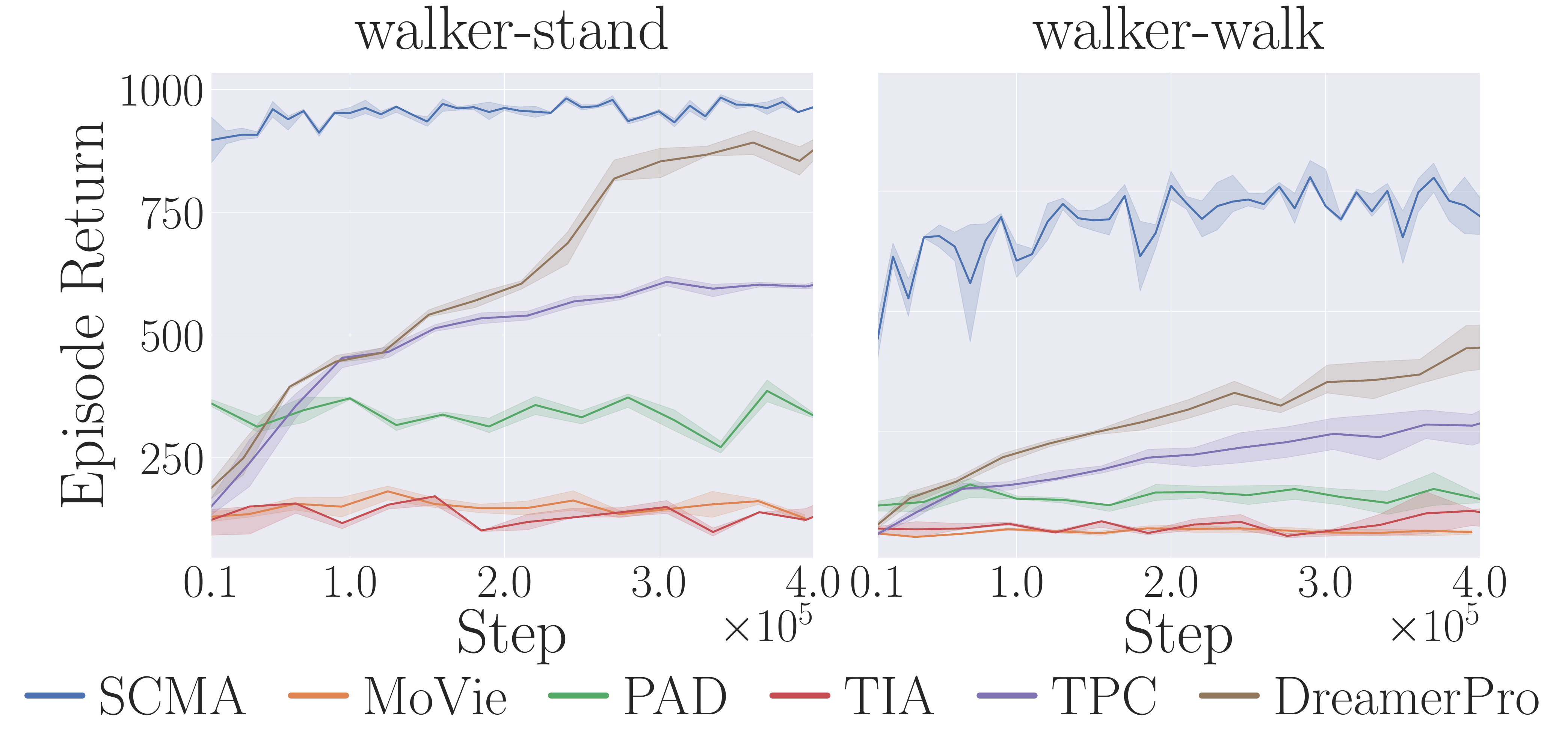}
\end{subfigure}
\vspace{-0.8em}
\caption{Performance curves of different algorithms in the $\mathrm{video\_hard}$ environment, where SCMA exhibits better final performance and sample efficiency.}
% Performance curves in the $\mathrm{video\_hard}$ environment. We report the performance curves of different algorithms in the $\mathrm{video\_hard}$ environment, where SCMA exhibits better final performance and sample efficiency.
\label{fig_video_hard_adapt}
\vspace{-1em}
\end{figure}

\subsection{Sample Efficiency in Visually Distracting Environments}

Accomplishing tasks with as few cluttered observations as possible is practically important to deploy the agent in distracting environments. Compared with other adaptation-based methods or training from scratch with task-induced methods~\citep{fu2021learning,deng2022dreamerpro}, the performance curves in Fig.~\ref{fig_video_hard_adapt} show that SCMA can achieve higher performance with much fewer downstream cluttered samples. Although we adapt the policy in $\mathrm{video\_hard}$ with $0.4$M steps, SCMA can achieve competitive performance with much fewer steps. We provide the wall clock time and adaptation steps for SCMA to reach $90\%$ of the final performance in Table~\ref{tabale_video_hard_wall_time} in the Appendix to show that SCMA obtains compelling results with only $10\%$ of total adaptation time-steps for most tasks.

\subsection{Real-world Robot Data}

With the rapid development of generative models, their potential to enhance real-world robotic controls has attracted significant attention. Recent works leverage video models to create future observations based on current environment observation and extract executable action sequences with inverse dynamics models (IDM)~\citep{Du2023LearningUP,Ko2023LearningTA}. However, the generated observations might still contain distractions if the input environment observation is cluttered, which imposes challenges to the IDM in making accurate action predictions. We show that SCMA can help IDM better predict the actions when handling cluttered observations. More details are included in Appendix~\ref{appendix_subsec_robot_setting}.

We manually collect real-world robot data with a Mobile ALOHA robot by performing an apple-grasping task with teleoperation. The IDM is trained with data collected in the normal setting and evaluated on data collected in $3$ distracting settings: 1) $\mathrm{fruit\_bg}$: various fruits are placed in the background. 2) $\mathrm{color\_bg}$: the scene is disrupted by a blue light. 3) $\mathrm{varying\_light}$: the lighting condition is intentionally changed. We provide the quantitative results in Table~\ref{table_robot_scma} in the Appendix and visualization in Fig.~\ref{fig_robot_visualization}. The results show that SCMA can effectively mitigate real-world distractions and thus has important implications for the practical deployment of robots.

\begin{figure}
    \centering
    \includegraphics[width=0.7\linewidth]{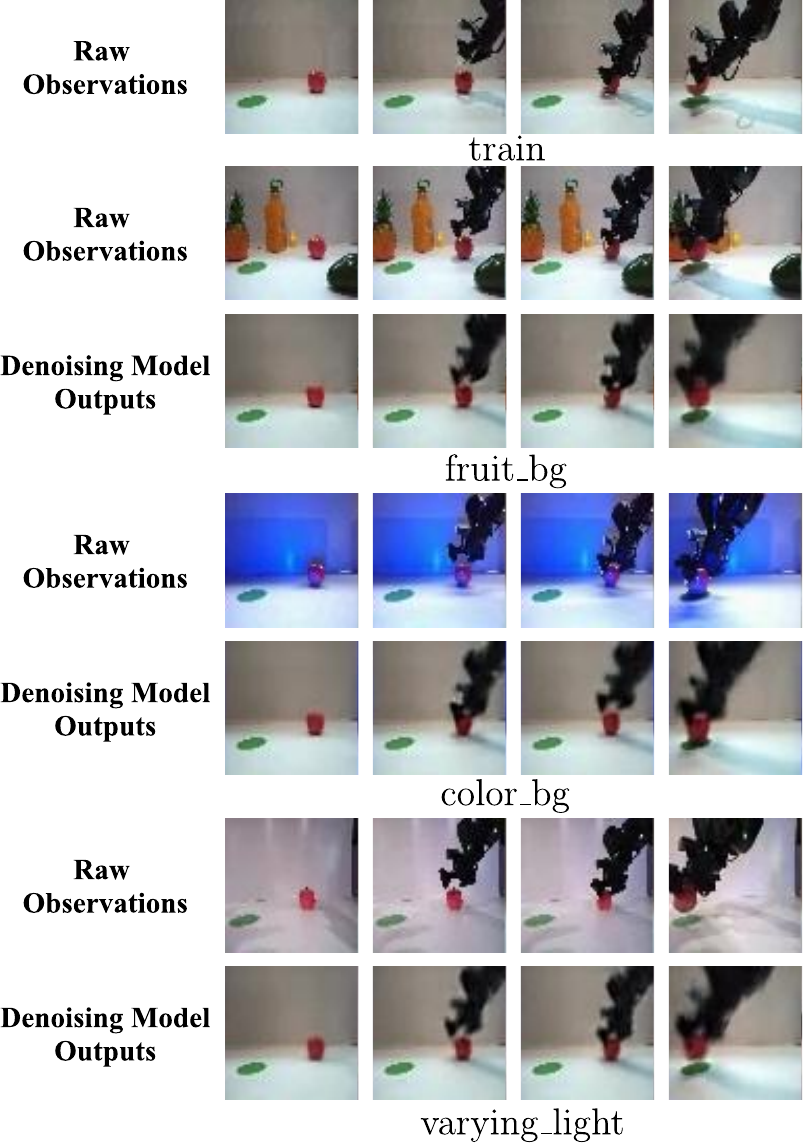}
    \vspace{-0.5em}
    \caption{Visualization of the raw observations and denoising model's outputs on real-world robot data.}
    \vspace{-1.2em}
    \label{fig_robot_visualization}
\end{figure}

\section{Conclusion and Discussion}
\label{sec_conclusion}

%  Discussion里可以加一些关于adapt速度的, 加速adapt。
% The developments in visual reinforcement learning have highlighted the need to fill the performance gap caused by visual distractions, which is a prerequisite for deploying learned policies into the real world

The ability to generalize across environments with various distractions is a long-standing goal in visual RL. In this work, we formalize the challenge as an unsupervised transferring problem and propose a novel method called self-consistent model-based adaptation (SCMA). SCMA adopts a policy-agnostic denoising model to mitigate distractions by transferring cluttered observations into clean ones. To optimize the denoising model in the absence of paired data, we propose an unsupervised distribution matching objective that regularizes the outputs of the denoising model to follow the distribution of clean observations, which can be estimated with a pre-trained world model. Experiments in challenging visual generalization benchmarks show that SCMA effectively reduces the performance gap caused by distractions and can boost the performance of various policies in a plug-and-play manner. Moreover, we validate the effectiveness of SCMA with real-world robot data, where SCMA effectively mitigates distractions and promotes better action predictions.

SCMA proposes a general model-based objective for adaptation under distractions, and we wish to further promote this direction by highlighting some limitations and future improvements. SCMA  pre-trains world models to estimate the action-conditioned distribution of clean observations. Including stronger world models like diffusion models~\citep{wang2023drivedreamer} may be a promising way to further promote the performance with complex robots or real-world tasks. Another potential improvement is to explore other types of signals that are invariant between clean and distracting environments, e.g. 3D structures of robots~\citep{driess2022reinforcement} or natural language description of tasks~\citep{Sumers2023DistillingIV}.

\newpage
\appendix
\section*{Ethical Statement}

The ability to neglect distractions is a prerequisite for the real-world application of visual reinforcement learning policies. This work aims to boost the visual robustness of learned agents through test-time adaptation with pre-trained world models, which might facilitate the deployment of intelligent agents. 
There are no serious ethical issues as it is basic research on reinforcement learning. We hope our work can inspire future research on designing robust agents under visual distractions.

\bibliographystyle{named}
\bibliography{ijcai25}

\begin{thebibliography}{}

\bibitem[\protect\citeauthoryear{Artetxe \bgroup \em et al.\egroup }{2017}]{artetxe2017unsupervised}
Mikel Artetxe, Gorka Labaka, Eneko Agirre, and Kyunghyun Cho.
\newblock Unsupervised neural machine translation.
\newblock {\em arXiv preprint arXiv:1710.11041}, 2017.

\bibitem[\protect\citeauthoryear{Baktashmotlagh \bgroup \em et al.\egroup }{2016}]{baktashmotlagh2016distribution}
Mahsa Baktashmotlagh, Mehrtash Har, Mathieu Salzmann, et~al.
\newblock Distribution-matching embedding for visual domain adaptation.
\newblock {\em Journal of Machine Learning Research}, 17(108):1--30, 2016.

\bibitem[\protect\citeauthoryear{Bertoin \bgroup \em et al.\egroup }{2022}]{Bertoin2022LookWY}
David Bertoin, Adil Zouitine, Mehdi Zouitine, and Emmanuel Rachelson.
\newblock Look where you look! saliency-guided q-networks for generalization in visual reinforcement learning.
\newblock In {\em Neural Information Processing Systems}, 2022.

\bibitem[\protect\citeauthoryear{Brohan \bgroup \em et al.\egroup }{2023}]{brohan2023can}
Anthony Brohan, Yevgen Chebotar, Chelsea Finn, Karol Hausman, Alexander Herzog, Daniel Ho, Julian Ibarz, Alex Irpan, Eric Jang, Ryan Julian, et~al.
\newblock Do as i can, not as i say: Grounding language in robotic affordances.
\newblock In {\em Conference on Robot Learning}, pages 287--318. PMLR, 2023.

\bibitem[\protect\citeauthoryear{Cao \bgroup \em et al.\egroup }{2018}]{cao2018unsupervised}
Yue Cao, Mingsheng Long, and Jianmin Wang.
\newblock Unsupervised domain adaptation with distribution matching machines.
\newblock In {\em Proceedings of the AAAI conference on artificial intelligence}, volume~32, 2018.

\bibitem[\protect\citeauthoryear{Caron \bgroup \em et al.\egroup }{2020}]{caron2020unsupervised}
Mathilde Caron, Ishan Misra, Julien Mairal, Priya Goyal, Piotr Bojanowski, and Armand Joulin.
\newblock Unsupervised learning of visual features by contrasting cluster assignments.
\newblock {\em Advances in Neural Information Processing Systems}, 33:9912--9924, 2020.

\bibitem[\protect\citeauthoryear{Chaplot \bgroup \em et al.\egroup }{2020}]{chaplot2020object}
Devendra~Singh Chaplot, Dhiraj~Prakashchand Gandhi, Abhinav Gupta, and Russ~R Salakhutdinov.
\newblock Object goal navigation using goal-oriented semantic exploration.
\newblock {\em Advances in Neural Information Processing Systems}, 33, 2020.

\bibitem[\protect\citeauthoryear{Chen \bgroup \em et al.\egroup }{2018}]{chen2018isolating}
Ricky~TQ Chen, Xuechen Li, Roger~B Grosse, and David~K Duvenaud.
\newblock Isolating sources of disentanglement in variational autoencoders.
\newblock {\em Advances in neural information processing systems}, 31, 2018.

\bibitem[\protect\citeauthoryear{Chen \bgroup \em et al.\egroup }{2024}]{chen2024focus}
Chao Chen, Jiacheng Xu, Weijian Liao, Hao Ding, Zongzhang Zhang, Yang Yu, and Rui Zhao.
\newblock Focus-then-decide: Segmentation-assisted reinforcement learning.
\newblock In {\em Proceedings of the AAAI Conference on Artificial Intelligence}, volume~38, pages 11240--11248, 2024.

\bibitem[\protect\citeauthoryear{Cheng}{1995}]{cheng1995mean}
Yizong Cheng.
\newblock Mean shift, mode seeking, and clustering.
\newblock {\em IEEE transactions on pattern analysis and machine intelligence}, 17(8):790--799, 1995.

\bibitem[\protect\citeauthoryear{Deng \bgroup \em et al.\egroup }{2022}]{deng2022dreamerpro}
Fei Deng, Ingook Jang, and Sungjin Ahn.
\newblock Dreamerpro: Reconstruction-free model-based reinforcement learning with prototypical representations.
\newblock In {\em International Conference on Machine Learning}, pages 4956--4975. PMLR, 2022.

\bibitem[\protect\citeauthoryear{Devo \bgroup \em et al.\egroup }{2020}]{devo2020towards}
Alessandro Devo, Giacomo Mezzetti, Gabriele Costante, Mario~L Fravolini, and Paolo Valigi.
\newblock Towards generalization in target-driven visual navigation by using deep reinforcement learning.
\newblock {\em IEEE Transactions on Robotics}, 36(5):1546--1561, 2020.

\bibitem[\protect\citeauthoryear{Ding \bgroup \em et al.\egroup }{2024}]{ding2024diffusionworldmodelfuture}
Zihan Ding, Amy Zhang, Yuandong Tian, and Qinqing Zheng.
\newblock Diffusion world model: Future modeling beyond step-by-step rollout for offline reinforcement learning, 2024.

\bibitem[\protect\citeauthoryear{Dirac}{1981}]{dirac1981principles}
Paul Adrien~Maurice Dirac.
\newblock {\em The principles of quantum mechanics}.
\newblock Number~27. Oxford university press, 1981.

\bibitem[\protect\citeauthoryear{Driess \bgroup \em et al.\egroup }{2022}]{driess2022reinforcement}
Danny Driess, Ingmar Schubert, Pete Florence, Yunzhu Li, and Marc Toussaint.
\newblock Reinforcement learning with neural radiance fields.
\newblock {\em Advances in Neural Information Processing Systems}, 35:16931--16945, 2022.

\bibitem[\protect\citeauthoryear{Du \bgroup \em et al.\egroup }{2023}]{Du2023LearningUP}
Yilun Du, Mengjiao Yang, Bo~Dai, Hanjun Dai, Ofir Nachum, Joshua~B. Tenenbaum, Dale Schuurmans, and P.~Abbeel.
\newblock Learning universal policies via text-guided video generation.
\newblock {\em ArXiv}, abs/2302.00111, 2023.

\bibitem[\protect\citeauthoryear{Fu \bgroup \em et al.\egroup }{2021}]{fu2021learning}
Xiang Fu, Ge~Yang, Pulkit Agrawal, and Tommi Jaakkola.
\newblock Learning task informed abstractions.
\newblock In {\em International Conference on Machine Learning}, pages 3480--3491. PMLR, 2021.

\bibitem[\protect\citeauthoryear{Ha \bgroup \em et al.\egroup }{2023}]{ha2023dream}
Jeongsoo Ha, Kyungsoo Kim, and Yusung Kim.
\newblock Dream to generalize: zero-shot model-based reinforcement learning for unseen visual distractions.
\newblock In {\em Proceedings of the AAAI Conference on Artificial Intelligence}, volume~37, pages 7802--7810, 2023.

\bibitem[\protect\citeauthoryear{Hafner \bgroup \em et al.\egroup }{2019a}]{hafner2019dream}
Danijar Hafner, Timothy Lillicrap, Jimmy Ba, and Mohammad Norouzi.
\newblock Dream to control: Learning behaviors by latent imagination.
\newblock {\em arXiv preprint arXiv:1912.01603}, 2019.

\bibitem[\protect\citeauthoryear{Hafner \bgroup \em et al.\egroup }{2019b}]{hafner2019learning}
Danijar Hafner, Timothy Lillicrap, Ian Fischer, Ruben Villegas, David Ha, Honglak Lee, and James Davidson.
\newblock Learning latent dynamics for planning from pixels.
\newblock In {\em International conference on machine learning}, pages 2555--2565. PMLR, 2019.

\bibitem[\protect\citeauthoryear{Hafner \bgroup \em et al.\egroup }{2023}]{hafner2023mastering}
Danijar Hafner, Jurgis Pasukonis, Jimmy Ba, and Timothy Lillicrap.
\newblock Mastering diverse domains through world models.
\newblock {\em arXiv preprint arXiv:2301.04104}, 2023.

\bibitem[\protect\citeauthoryear{Hansen and Wang}{2021}]{hansen2021generalization}
Nicklas Hansen and Xiaolong Wang.
\newblock Generalization in reinforcement learning by soft data augmentation.
\newblock In {\em 2021 IEEE International Conference on Robotics and Automation (ICRA)}, pages 13611--13617. IEEE, 2021.

\bibitem[\protect\citeauthoryear{Hansen \bgroup \em et al.\egroup }{2020}]{Hansen2020SelfSupervisedPA}
Nicklas Hansen, Yu~Sun, P.~Abbeel, Alexei~A. Efros, Lerrel Pinto, and Xiaolong Wang.
\newblock Self-supervised policy adaptation during deployment.
\newblock {\em ArXiv}, abs/2007.04309, 2020.

\bibitem[\protect\citeauthoryear{Hansen \bgroup \em et al.\egroup }{2021}]{hansen2021stabilizing}
Nicklas Hansen, Hao Su, and Xiaolong Wang.
\newblock Stabilizing deep q-learning with convnets and vision transformers under data augmentation.
\newblock 2021.

\bibitem[\protect\citeauthoryear{Higgins \bgroup \em et al.\egroup }{2017}]{higgins2017beta}
Irina Higgins, Loic Matthey, Arka Pal, Christopher~P Burgess, Xavier Glorot, Matthew~M Botvinick, Shakir Mohamed, and Alexander Lerchner.
\newblock beta-vae: Learning basic visual concepts with a constrained variational framework.
\newblock {\em ICLR (Poster)}, 3, 2017.

\bibitem[\protect\citeauthoryear{Jaderberg \bgroup \em et al.\egroup }{2015}]{jaderberg2015spatial}
Max Jaderberg, Karen Simonyan, Andrew Zisserman, et~al.
\newblock Spatial transformer networks.
\newblock {\em Advances in neural information processing systems}, 28, 2015.

\bibitem[\protect\citeauthoryear{Ko \bgroup \em et al.\egroup }{2023}]{Ko2023LearningTA}
Po-Chen Ko, Jiayuan Mao, Yilun Du, Shao-Hua Sun, and Josh Tenenbaum.
\newblock Learning to act from actionless videos through dense correspondences.
\newblock {\em ArXiv}, abs/2310.08576, 2023.

\bibitem[\protect\citeauthoryear{Lachaux \bgroup \em et al.\egroup }{2020}]{lachaux2020unsupervised}
Marie-Anne Lachaux, Baptiste Roziere, Lowik Chanussot, and Guillaume Lample.
\newblock Unsupervised translation of programming languages.
\newblock {\em arXiv preprint arXiv:2006.03511}, 2020.

\bibitem[\protect\citeauthoryear{Li \bgroup \em et al.\egroup }{2023}]{Li2023CrossLocoHM}
Tianyu Li, Hyunyoung Jung, Matthew Gombolay, Yong~Kwon Cho, and Sehoon Ha.
\newblock Crossloco: Human motion driven control of legged robots via guided unsupervised reinforcement learning.
\newblock {\em ArXiv}, abs/2309.17046, 2023.

\bibitem[\protect\citeauthoryear{Li \bgroup \em et al.\egroup }{2024}]{Li2024Think2DriveER}
Qifeng Li, Xiaosong Jia, Shaobo Wang, and Junchi Yan.
\newblock Think2drive: Efficient reinforcement learning by thinking in latent world model for quasi-realistic autonomous driving (in carla-v2).
\newblock 2024.

\bibitem[\protect\citeauthoryear{Liu \bgroup \em et al.\egroup }{2023}]{liu2023cross}
Xin Liu, Yaran Chen, Haoran Li, Boyu Li, and Dongbin Zhao.
\newblock Cross-domain random pre-training with prototypes for reinforcement learning.
\newblock {\em arXiv preprint arXiv:2302.05614}, 2023.

\bibitem[\protect\citeauthoryear{Nair \bgroup \em et al.\egroup }{2022}]{nair2022r3m}
Suraj Nair, Aravind Rajeswaran, Vikash Kumar, Chelsea Finn, and Abhinav Gupta.
\newblock R3m: A universal visual representation for robot manipulation.
\newblock {\em arXiv preprint arXiv:2203.12601}, 2022.

\bibitem[\protect\citeauthoryear{Nguyen \bgroup \em et al.\egroup }{2021}]{nguyen2021temporal}
Tung~D Nguyen, Rui Shu, Tuan Pham, Hung Bui, and Stefano Ermon.
\newblock Temporal predictive coding for model-based planning in latent space.
\newblock In {\em International Conference on Machine Learning}, pages 8130--8139. PMLR, 2021.

\bibitem[\protect\citeauthoryear{Pan \bgroup \em et al.\egroup }{2022}]{pan2022isolating}
Minting Pan, Xiangming Zhu, Yunbo Wang, and Xiaokang Yang.
\newblock Isolating and leveraging controllable and noncontrollable visual dynamics in world models.
\newblock {\em arXiv preprint arXiv:2205.13817}, 2022.

\bibitem[\protect\citeauthoryear{Shah \bgroup \em et al.\egroup }{2023}]{shah2023gnm}
Dhruv Shah, Ajay Sridhar, Arjun Bhorkar, Noriaki Hirose, and Sergey Levine.
\newblock Gnm: A general navigation model to drive any robot.
\newblock In {\em 2023 IEEE International Conference on Robotics and Automation (ICRA)}, pages 7226--7233. IEEE, 2023.

\bibitem[\protect\citeauthoryear{Shridhar \bgroup \em et al.\egroup }{2023}]{shridhar2023perceiver}
Mohit Shridhar, Lucas Manuelli, and Dieter Fox.
\newblock Perceiver-actor: A multi-task transformer for robotic manipulation.
\newblock In {\em Conference on Robot Learning}, pages 785--799. PMLR, 2023.

\bibitem[\protect\citeauthoryear{Sumers \bgroup \em et al.\egroup }{2023}]{Sumers2023DistillingIV}
Theodore~R. Sumers, Kenneth Marino, Arun Ahuja, Rob Fergus, and Ishita Dasgupta.
\newblock Distilling internet-scale vision-language models into embodied agents.
\newblock {\em ArXiv}, abs/2301.12507, 2023.

\bibitem[\protect\citeauthoryear{Tassa \bgroup \em et al.\egroup }{2018}]{tassa2018deepmind}
Yuval Tassa, Yotam Doron, Alistair Muldal, Tom Erez, Yazhe Li, Diego de~Las Casas, David Budden, Abbas Abdolmaleki, Josh Merel, Andrew Lefrancq, et~al.
\newblock Deepmind control suite.
\newblock {\em arXiv preprint arXiv:1801.00690}, 2018.

\bibitem[\protect\citeauthoryear{Tomar \bgroup \em et al.\egroup }{2021}]{tomar2021learning}
Manan Tomar, Utkarsh~A Mishra, Amy Zhang, and Matthew~E Taylor.
\newblock Learning representations for pixel-based control: What matters and why?
\newblock {\em arXiv preprint arXiv:2111.07775}, 2021.

\bibitem[\protect\citeauthoryear{Wang \bgroup \em et al.\egroup }{2021}]{wang2021survey}
Feng Wang, Lianmeng Jiao, and Quan Pan.
\newblock A survey on unsupervised transfer clustering.
\newblock In {\em 2021 40th Chinese Control Conference (CCC)}, pages 7361--7365. IEEE, 2021.

\bibitem[\protect\citeauthoryear{Wang \bgroup \em et al.\egroup }{2022}]{wang2022denoised}
Tongzhou Wang, Simon Du, Antonio Torralba, Phillip Isola, Amy Zhang, and Yuandong Tian.
\newblock Denoised mdps: Learning world models better than the world itself.
\newblock In {\em International Conference on Machine Learning}, pages 22591--22612. PMLR, 2022.

\bibitem[\protect\citeauthoryear{Wang \bgroup \em et al.\egroup }{2023}]{wang2023drivedreamer}
Xiaofeng Wang, Zheng Zhu, Guan Huang, Xinze Chen, and Jiwen Lu.
\newblock Drivedreamer: Towards real-world-driven world models for autonomous driving.
\newblock {\em arXiv preprint arXiv:2309.09777}, 2023.

\bibitem[\protect\citeauthoryear{Yang \bgroup \em et al.\egroup }{2024}]{yang2024movie}
Sizhe Yang, Yanjie Ze, and Huazhe Xu.
\newblock Movie: Visual model-based policy adaptation for view generalization.
\newblock {\em Advances in Neural Information Processing Systems}, 36, 2024.

\bibitem[\protect\citeauthoryear{Ying \bgroup \em et al.\egroup }{2024}]{ying2024peac}
Chengyang Ying, Zhongkai Hao, Xinning Zhou, Xuezhou Xu, Hang Su, Xingxing Zhang, and Jun Zhu.
\newblock Peac: Unsupervised pre-training for cross-embodiment reinforcement learning.
\newblock {\em arXiv preprint arXiv:2405.14073}, 2024.

\bibitem[\protect\citeauthoryear{Yuan \bgroup \em et al.\egroup }{2022}]{yuan2022pre}
Zhecheng Yuan, Zhengrong Xue, Bo~Yuan, Xueqian Wang, Yi~Wu, Yang Gao, and Huazhe Xu.
\newblock Pre-trained image encoder for generalizable visual reinforcement learning.
\newblock {\em Advances in Neural Information Processing Systems}, 35:13022--13037, 2022.

\bibitem[\protect\citeauthoryear{Yuan \bgroup \em et al.\egroup }{2024}]{yuan2024rl}
Zhecheng Yuan, Sizhe Yang, Pu~Hua, Can Chang, Kaizhe Hu, and Huazhe Xu.
\newblock Rl-vigen: A reinforcement learning benchmark for visual generalization.
\newblock {\em Advances in Neural Information Processing Systems}, 36, 2024.

\bibitem[\protect\citeauthoryear{Zhao \bgroup \em et al.\egroup }{2022}]{Zhao2022EGSDEUI}
Min Zhao, Fan Bao, Chongxuan Li, and Jun Zhu.
\newblock Egsde: Unpaired image-to-image translation via energy-guided stochastic differential equations.
\newblock {\em ArXiv}, abs/2207.06635, 2022.

\bibitem[\protect\citeauthoryear{Zhou \bgroup \em et al.\egroup }{2017}]{zhou2017places}
Bolei Zhou, Agata Lapedriza, Aditya Khosla, Aude Oliva, and Antonio Torralba.
\newblock Places: A 10 million image database for scene recognition.
\newblock {\em IEEE transactions on pattern analysis and machine intelligence}, 40(6):1452--1464, 2017.

\bibitem[\protect\citeauthoryear{Zhu \bgroup \em et al.\egroup }{2017}]{Zhu2017UnpairedIT}
Jun-Yan Zhu, Taesung Park, Phillip Isola, and Alexei~A. Efros.
\newblock Unpaired image-to-image translation using cycle-consistent adversarial networks.
\newblock {\em 2017 IEEE International Conference on Computer Vision (ICCV)}, pages 2242--2251, 2017.

\bibitem[\protect\citeauthoryear{Zhu \bgroup \em et al.\egroup }{2020}]{zhu2020robosuite}
Yuke Zhu, Josiah Wong, Ajay Mandlekar, Roberto Mart{\'\i}n-Mart{\'\i}n, Abhishek Joshi, Soroush Nasiriany, and Yifeng Zhu.
\newblock robosuite: A modular simulation framework and benchmark for robot learning.
\newblock {\em arXiv preprint arXiv:2009.12293}, 2020.

\end{thebibliography}

\newpage
\section{Theoretical Analyses}
In this section, we will provide proof of all our theoretical results in detail.

\subsection{Noisy Partially-Observed Markov Decision Process}
\label{appendix_subsec_npomdp}
For NPOMDP $\mathcal{M}_n = \langle \mathcal{S}, \mathcal{O}, \mathcal{A}, \mathcal{T}, \mathcal{R}, \gamma, \rho_0, f_n\rangle$, the action-conditioned joint distribution is defined as following:
\begin{equation*}
\resizebox{.95\linewidth}{!}{$
            \displaystyle
            \begin{aligned} 
                &p(o_{1:T}, o^{n}_{1:T}, r_{1:T}|a_{1:T}) \coloneqq \\
                &\int\prod\limits_{t=1}^{T} p(o^n_t|o_t) p(o_{t}|s_{\leq t}, a_{<t})p(r_{t}|s_{\leq t}, a_{<t})p(s_{t}|s_{<t},a_{<t})\mathrm{d}s_{1:T},
            \end{aligned} 
$}
\end{equation*}
where $p(o^n_t|o_t)=\delta(o^n_t-f_n(o_t))$ is the noising distribution with $\delta(\cdot)$ being the Dirac delta function. It should be noted that all the noisy/denoising distributions are assumed to be independent and identically distributed, i.e. $p(o^n_{1:T}|o_{1:T})=\prod_t p(o^n_t|o_t)$, $ q(o_{1:T}|o^n_{1:T})=\prod_t q(o_t|o^n_t)$. Leveraging the Bayes' rule, we have that:
\begin{equation*}
    p(o_t|o^n_t) = \frac{p(o^n_t|o_t)p(o_t)}{\sum_{o'_t}p(o^n_t|o'_t)p(o'_t)}.
\end{equation*}

In the theoretical analysis, we assume $f_n$ is an injective function. We will next explain why this is a reasonable assumption in practical visual generalization settings. Following previous works~\citep{fu2021learning,Bertoin2022LookWY,yuan2022pre,yuan2024rl}, visual generalization involves variations in task-irrelevant factors, such as colors, backgrounds, and lighting conditions. However, task-relevant factors remain untouched, such as the robot's pose. Therefore, the noise function $f_n$ should not map two different clean observations into the same cluttered observation, i.e., $f_n$ is injective.

For the simplicity of the following derivations, we now redefine the observation spaces of clean and cluttered observations separately:
\begin{equation*}
    \begin{aligned}
        \mathcal{O}^c &= \left \{o  |o\in \mathcal{O};\exists t, p(o_t=o|a_{1:T})>0\right \}\\
        \mathcal{O}^n &= \left \{o  |o \in \mathcal{O};\exists t, p(o_t=o|a_{1:T})>0 \right \}.
    \end{aligned}
\end{equation*}

Generally speaking, $\mathcal{O}^c$ and $\mathcal{O}^n$ only contain observations that might occur. By redefining $f_n: \mathcal{O^c}\mapsto \mathcal{O}^n$, $f_n$ is now an bijective function. We denote the inverse of $f_n$ as $f_n^{-1}(o^n_t)= f^{-1}_n(f_n(o_t)) = o_t$.

With $p(o^n_t|o_t)=\delta(o^n_t-f_n(o_t))$, we can show that:
\begin{equation*}
    p(o_t|o^n_t) =\left \{ \begin{aligned} 1, & \, o_t = f^{-1}_n(o^n_t) \\ 0, & \, \text{otherwise}\end{aligned} \right . ,
\end{equation*}
which means the posterior denoising distribution $p(o_t|o^n_t)$ is also a Dirac distribution, i.e. $p(o_t|o^n_t)=\delta(o_t-f^{-1}_n(o^n_t))$.

\subsection{Mitigate Distractions with Unsupervised Distribution Matching}
\label{appendix_subsec_distribution}

\paragraph{Homogeneous Noise Function} From the definition of homogeneous noise functions provided (Def.~\ref{def_homo_noise}), we can show that homogeneous noise functions are theoretically indistinguishable in the unsupervised setting. Without loss of generality, we only consider two random variables $(o, o^n)$, omitting the time subscript $t$ and action condition $a_{1:T}$. Given a clean marginal distribution $p(o)$, the noise function $f_n$ specifies the conditional distribution $p(o^n|o)=\delta(o^n-f_n(o))$, which in turn defines a corresponding joint distribution $p(o^n,o)$ and cluttered marginal distribution $p(o^n)$. We then define $\mathcal{H}_{f_n}^p = \left \{ f_{n_i} | f_{n_i}\equiv_p f_{n}\right \}$ to be the set of homogeneous noise functions of $f_n$ under $p(o)$.
For homogeneous noise functions, they all share the same marginal distributions $p(o)$ and $p(o^n)$ but have different conditional distributions $p(o^n|o)$ and joint distributions $p(o^n,o)$. In the unsupervised setting, we can only collect samples to estimate $p(o)$ and $p(o^n)$ separately, which makes it impossible to distinguish between joint distributions with the same marginal distributions. Therefore, it is impossible to distinguish between homogeneous noise functions unless leveraging additional assumptions.

% if we need more proof, we can prove that there are more than one joint probability with the same marginal probability 

% Since there are multiple joint distributions $p(o^n,o)$ with the same marginal distributions $p(o)$ and $p(o^n)$, each of them corresponds to a noise function, and together they form the set of homogeneous noise functions.

\paragraph{Unsupervised Distribution Matching}

Due to the lack of paired data between clean and cluttered observations, we address this challenge with unsupervised distribution matching. An important insight is that although distracting environments have unknown visual variations, the task-relevant objects still follow the same latent dynamics as in clean environments. Specifically, given action sequences $a_{1:T}$, we can collect cluttered observations $o^n_{1:T}$ from the distracting environments (i.e. $p(o^n_{1:T}|a_{1:T})$. The corresponding clean observations $o_{1:T}$, although unobservable, still follow $p(o_{1:T}|a_{1:T})$ as in clean environments. Compared to traditional unsupervised transferring in computer vision that operates on static datasets, we can obtain a certain level of control over the distribution of observations by selecting specific action sequences.

Therefore, a natural way to optimize the denoising model is to align the distribution between the transferred observations and the clean observations:
\begin{equation*}
% \label{eq_surrogate_kl}
\begin{aligned}
    \mathcal{L}'_{\mathrm{KL}} = \mathrm{D}_{\mathrm{KL}}\left (\mathbb{E}_{p(o^n_{1:T}|a_{1:T})}[q(o_{1:T}|o^n_{1:T})] \big\Vert p(o_{1:T}|a_{1:T})\right ).
\end{aligned}
\end{equation*}

However, $\mathcal{L}'_{\mathrm{KL}}$ is not directly optimizable as it is non-trivial to estimate the distribution of transferred observations $\mathbb{E}_{p(o^n_{1:T}|a_{1:T})}[q(o_{1:T}|o^n_{1:T})]$.

To address this problem, we additionally introduce a learnable noisy distribution $q(o^n_t|o_t)$ and extend $\mathcal{L}'_{\mathrm{KL}}$ to the following $\mathcal{L}^{\mathrm{joint}}_{\mathrm{KL}}$:
\begin{equation}
\begin{aligned}
    \mathcal{L}_{\mathrm{KL}} =& \mathrm{D}_{\mathrm{KL}} \Big(p(o^n_{1:T}|a_{1:T})q(o_{1:T}|o^n_{1:T}) \\
    &\quad\quad\quad\quad\quad\quad \big\Vert p(o_{1:T}|a_{1:T})q(o^n_{1:T}|o_{1:T})\Big).
\end{aligned}
\end{equation}

With $q(o^n_{1:T}|o_{1:T})$ being optimizable, we demonstrate that $\mathcal{L}_{\mathrm{KL}}$ is equivalent to $\mathcal{L}'_{\mathrm{KL}}$ in the sense that the optimal denoising distribution $q^*(o_{1:T}|o^n_{1:T})$ is identical for both objectives.
\begin{equation}
    \mathop{\arg\min}_{q(o_{1:T}|o^n_{1:T})}\; \min_{q(o^n_{1:T}|o_{1:T})}\mathcal{L}_{\mathrm{KL}}= \mathop{\arg\min}_{q(o_{1:T}|o^n_{1:T})}\; \mathcal{L}'_{\mathrm{KL}}.
\end{equation}

\begin{proof}
Let $q^1(o_{1:T}|o^n_{1:T}) \coloneqq \mathop{\arg\min}\limits_{q(o_{1:T}|o^n_{1:T})}\; \min\limits_{q(o^n_{1:T}|o_{1:T})}\mathcal{L}_{\mathrm{KL}}$, $q^2(o_{1:T}|o^n_{1:T}) \coloneqq \mathop{\arg\min}\limits_{q(o_{1:T}|o^n_{1:T})}\; \mathcal{L}'_{\mathrm{KL}}$. The goal is to show that $q^1(o_{1:T}|o^n_{1:T})$ also minimizes $\mathcal{L}'_{\mathrm{KL}}$ and $q^2(o_{1:T}|o^n_{1:T})$ also minimizes $\mathcal{L}_{\mathrm{KL}}$.

It is easy to show that $q^1(o_{1:T}|o^n_{1:T})$ minimizes $\mathcal{L}'_{\mathrm{KL}}$. According to the properties of KL-divergence, $q^1(o_{1:T}|o^n_{1:T})$ reaches the minimum point of $\mathcal{L}'_{\mathrm{KL}}$ if and only if $p(o^n_{1:T}|a_{1:T})q^1(o_{1:T}|o^n_{1:T})=p(o_{1:T}|a_{1:T})q(o^n_{1:T}|o_{1:T})$. Therefore, $\mathbb{E}_{p(o^n_{1:T}|a_{1:T})}[q^1(o_{1:T}|o^n_{1:T})=p(o_{1:T}|a_{1:T})$, which means $\mathcal{L}'_{\mathrm{KL}}=0$.

\vspace{0.1em}

To show that $q^2(o_{1:T}|o^n_{1:T})$ also minimizes $\mathcal{L}_{\mathrm{KL}}$, we only need to show that the following expression is a valid distribution:
$$\frac{p(o^n_{1:T}|a_{1:T})q^2(o_{1:T}|o^n_{1:T})}{p(o_{1:T}|a_{1:T})}.$$

Since $q^2(o_{1:T}|o^n_{1:T})$ minimizes $\mathcal{L}'_{\mathrm{KL}}$, it follows that $\mathbb{E}_{p(o^n_{1:T}|a_{1:T})}[q^2(o_{1:T}|o^n_{1:T})] = p(o_{1:T}|a_{1:T})$. Therefore, we have:

\resizebox{0.95\linewidth}{!}{$
    \displaystyle{
    \sum\limits_{o^n_{1:T}}\frac{p(o^n_{1:T}|a_{1:T})q^2(o_{1:T}|o^n_{1:T})}{p(o_{1:T}|a_{1:T})} = \frac{p(o_{1:T}|a_{1:T})}{p(o_{1:T}|a_{1:T})}=1.
    }
$}

Letting $q(o^n_{1:T}|o_{1:T})=\frac{p(o^n_{1:T}|a_{1:T})q^2(o_{1:T}|o^n_{1:T})}{p(o_{1:T}|a_{1:T})}$, it is easy to see that $\mathcal{L}_{\mathrm{KL}}=0$.
Thus, we have proven that $\mathcal{L}'_{\mathrm{KL}}$ and $\mathcal{L}_{\mathrm{KL}}$ share the same optimal denoising distribution $q^*(o_{1:T}|a_{1:T})$.
\end{proof}

To simplify the calculation, we further show that $\mathcal{L}_{\mathrm{KL}}$ can be derived into the following objective:
\begin{equation}
\label{eq_adaptation_raw_1}
\resizebox{\linewidth}{!}{$
    \displaystyle{
        \begin{aligned}
        \mathcal{L}_{\mathrm{KL}} &= \mathrm{D}_{\mathrm{KL}}\Big(p(o^n_{1:T}|a_{1:T})q(o_{1:T}|o^n_{1:T}) \\
        &\quad\quad\quad\quad\quad\quad \big\Vert p(o_{1:T}|a_{1:T})q(o^n_{1:T}|o_{1:T})\Big)\\
        &= -H\big (p(o^n_{1:T}|a_{1:T})\big) + \mathbb{E}_{p(o^n_{1:T}|a_{1:T})q(o_{1:T}|o^n_{1:T})}\big [ \\
        &\quad\quad \log q(o_{1:T}|o^n_{1:T}) - \log p(o_{1:T}|a_{1:T}) - \log q(o^n_{1:T}|o_{1:T}) \big ] \\
        % &= \mathbb{E}_{p(o^n_{1:T}|a_{1:T})} \Big[\mathrm{D}_{\mathrm{KL}}\big (q(o_{1:T}|o^n_{1:T})\|p(o_{1:T}|a_{1:T})\big) \\
        % &\quad\quad -\mathbb{E}_{q(o_{1:T}|o^n_{1:T})}[\log q(o^n_{1:T}|o_{1:T})]\Big ] -H\big (p(o^n_{1:T}|a_{1:T})\big)\\
        &\stackrel{(*)}{\approx} \mathbb{E}_{p(o^n_{1:T}|a_{1:T})}\mathbb{E}_{q(o_{1:T}|o^n_{1:T})} \Big [ -\log p(o_{1:T}|a_{1:T}) \\
        &\quad\quad-\log q(o^n_{1:T}|o_{1:T}) \Big] - \underbrace{H\big (p(o^n_{1:T}|a_{1:T})\big)}_{\text{constant}}.
        \end{aligned}
    }
$}
\end{equation}
In $(*)$, as $q(o_{1:T}|o^n_{1:T})$ is a Dirac distribution, i.e. $q(o_{1:T}|o^n_{1:T})=\prod_{t}q(o_t|o^n_t)=\prod_{t}\delta(o_t-m_{\mathrm{de}}(o_t^n))$, we have
\begin{equation*}
\begin{aligned}
    & \mathbb{E}_{p(o^n_{1:T}|a_{1:T})q(o_{1:T}|o^n_{1:T})}\big [ \log q(o_{1:T}|o^n_{1:T}) \big ] \\
    = &\mathbb{E}_{p(o^n_{1:T}|a_{1:T})}\bigg[ \prod_{t}\delta(o_t-m_{\mathrm{de}}(o_t^n)) \log \delta(o_t-m_{\mathrm{de}}(o_t^n)) \bigg] \\
    =& 0.
\end{aligned}
\end{equation*}
% \begin{equation*}
%     \begin{aligned}
%         &\mathrm{D}_{\mathrm{KL}}\big (q(o_{1:T}|o^n_{1:T})\|p(o_{1:T}|a_{1:T})\big ) \\
%         =& -\log p\big (m_{\mathrm{de}}(o^n_1), \cdots, m_{\mathrm{de}}(o^n_T)|a_{1:T}\big )\\
%         =& \mathbb{E}_{q(o_{1:T}|o^n_{1:T}})\left [ -\log p(o_{1:T}|a_{1:T})  \right ].
%     \end{aligned}
% \end{equation*}
Moreover, $-H\big (p(o^n_{1:T}|a_{1:T})\big)$ in the above objective is a constant, which is denoted as $C$ in the manuscript. 

\subsection{Optimality Analysis}
\label{appendix_subsec_optimality}
We provide detailed proof of Theorem~\ref{theorem_homo} in this section. For the notation simplicity, we only consider two random variables $(o,o^n)$, omitting the time subscript $t$ and action condition $a_{1:T}$. 

Given a clean marginal distribution $p(o)$ and noise function $f_n$, it defines the joint distribution $p(o^n,o)=p(o)p_n(o^n|o)$ and cluttered marginal distribution $p_n(o^n)=\sum_op(o^n,o)$. Leveraging the Bayes' rule, the posterior distribution $p(o|o^n)$ is also defined as $p(o|o^n)=\frac{p(o^n,o)}{p_n(o^n)}$. While we sometimes simplify $p_{o^n}(o^n)=p_n(o^n)$ as $p(o^n)$, we explicitly preserve the subscript to avoid confusion here.

Following the redefinition in Appendix~\ref{appendix_subsec_npomdp}, we have that $p(o)>0, p_n(o^n)>0, \forall o\in \mathcal{O}^c, o^n\in\mathcal{O}^n$. The redefinition ensures that only observations that might occur are considered. Assume that $|\mathcal{O}^c|=|\mathcal{O}^n|=N$, an critical insight is that $p_n(o_n)$ is a permutation of $p(o)$. Specifically, let $P\coloneqq[ p(o_1),\cdots,p(o_N)]$, $P^n\coloneqq[ p_n(o^n_1),\cdots,p_n(o^n_N)]$, we can transform $P$ to $P^n$ with permutation. This is easy to verify as we have $p(o)=p_n(f_n(o)),\forall o \in \mathcal{O}^c$.

We denote the optimal denoising distribution $q^*(o|o^n)$ and optimal noising distribution $q^*(o^n|o)$ as:
\begin{equation*}
    (q^*(o|o^n), q^*(o^n|o)) = \mathop{\arg\min}\limits_{q(o|o^n),q(o^n|o)} \mathcal{L}_{\mathrm{KL}}.
\end{equation*}

As we implement $q(o|o^n)=\delta(o-m_{\mathrm{de}}(o^n))$, $q(o^n|o)=\delta(o^n-m_{\mathrm{n}}(o))$. $q^*(o|o^n), q^*(o^n|o)$ are also constrained to be Dirac distributions. It should be noted that this constraint does not affect the minimum of $\mathcal{L}_{\mathrm{KL}}$ as we can set $q(o^n|o)=p(o^n|o),q(o|o^n)=p(o|o^n)$ so that $\mathcal{L}_{\mathrm{KL}}=0$.

To prove Theorem~\ref{theorem_homo}, we need to prove: 1) any optimal denoising distribution $q^*(o|o^n)$ is the posterior denoising distribution of a noise function that is homogeneous to $f_n$. 2) the posterior denoising distribution of a noise function that is homogeneous to $f_n$ minimizes $\mathcal{L}_{\mathrm{KL}}$.

\begin{lemma}
\label{lemma_injective}
    The optimal denoising/noising distribution is Dirac distribution with an injective function, i.e. $q^*(o|o^n)=\delta(o-m^*_{\mathrm{de}}(o^n))$, $q^*(o^n|o)=\delta(o^n-m^*_n(o))$, where $m^*_{\mathrm{de}}$ and $m^*_n$ are injective functions.
\end{lemma}
\begin{proof}

 According to the properties of KL-divergence, $\mathcal{L}_{\mathrm{KL}}$ reaches the minimum if and only if the distributions are identical, i.e.:
\begin{equation}
\label{eq_simple_kl}
    p_n(o^n)q^*(o|o^n)=p(o)q^*(o^n|o).
\end{equation}

We first show that $m^*_{\mathrm{de}}$ is injective. It is easy to observe that $\sum_{o^n}p_n(o^n)q^*(o|o^n)=\sum_{o^n}p(o)q^*(o^n|o)=p(o)$. As mentioned above, $p_n(o^n)$ can be viewed as a permutation of $p(o)$. Therefore, if there exists $(o^n_i, o^n_j)$ such that $m^*_{\mathrm{de}}(o^n_i)=m^*_{\mathrm{de}}(o^n_j)$, there must exist an $o$ with $p(o) = 0$, which conflicts with the assumption that $p(o)>0, \forall o$. Similarly, we can prove that $m^*_n$ is injective. 
\end{proof}

We first show that the optimal denoising distribution $q^*(o|o^n)$ is the posterior denoising distribution of a noise function that is homogeneous to $f_n$. Using Lemma~\ref{lemma_injective}, we have that $m^*_n$ is an injective function, which means that $m^*_n$ can also be viewed as a noise function. Given Eq.~\ref{eq_simple_kl}, it follows that $\sum_op(o)q^*(o^n|o)=\sum_op(o)\delta(o^n-m^*_n(o))=\sum_o p_n(o^n)q^*(o|o^n)=p_n(o^n)$, i.e. $m^*_n$ is a noise function that is homogeneous to $f_n$. We then have:
$$
q^*(o|o^n) = \frac{p(o)q^*(o^n|o)}{p_n(o^n)}=\frac{p(o)q^*(o^n|o)}{\sum_op(o)q^*(o^n|o)}.
$$

Therefore, we have shown that the optimal denoising distribution is the posterior denoising distribution of $m^*_{n}$, which is a noise function that is homogeneous to $f_n$.

Then we show that the posterior denoising distribution of a noise function that is homogeneous to $f_n$ minimizes $\mathcal{L}_{\mathrm{KL}}$. Let $f_{n_i}$ denote a noise function that is homogeneous to $f_n$, we have $\sum_op(o)\delta(o^n-f_{n_i}(o))=p_n(o^n)$. We further obtain that:
$$
\sum\limits_o\frac{p(o)\delta(o^n-f_{n_i}(o))}{p_n(o^n)}=\frac{p_n(o^n)}{p_n(o^n)} = 1,
$$
which means $\frac{p(o)\delta(o^n-f_{n_i}(o))}{p_n(o^n)}$ is a valid distribution. By choosing $q(o|o^n)=\frac{p(o)\delta(o^n-f_{n_i}(o))}{p_n(o^n)}, q(o^n|o)=\delta(o^n-f_{n_i}(o))$, we can derive that $\mathcal{L}_{\mathrm{KL}}=0$.

\paragraph{The Number of Homogeneous Noise Functions} Assume the clean observation $o$ has $N$ different possible values $\mathbf{o}=\{o_1, \cdots, o_N\}$, and the noise function $f_n$ maps them to $N$ different cluttered observations $\mathbf{o}^n=\{o^n_1, \cdots, o^n_N\}$. The marginal probability satisfies that $p(o_i) = p_n(o^n_i)$.

Therefore, the number of homogeneous noise functions should equal the number of different mapping $f$ between $\mathbf{o}$ and $\mathbf{o}^n$, so that $p_n(f_n(o))=p_n(f(o))$ for any $o$. Consequently, the number of homogeneous noise functions is determined by the number of different $o_i$ with the same probability $p(o_i)$. Specifically, assuming that $p(o)$ has $M$ different probabilities, each probability corresponds to $K_j$ different observations. In other words:
$$
\mathbf{o} = \{\underbrace{(o_1,\cdots,o_{K_1})}_{p(o_1)=\cdots=p(o_{K_1)}}, \cdots\}, \sum_{j=1}^M K_j = N.
$$

Then the number of homogeneous noise functions is $\prod\limits_j K_j!$.

\paragraph{Reducing Homogeneous Noise Functions with Rewards}

We illustrate why it is feasible to reduce the number of homogeneous noise functions with rewards. Suppose that we only have four possible scenarios in clean environments: $p(o_1,r_1)= p(o_2,r_2)=0.5, p(o_1,r_2)=p(o_2,r_1)=0$. And the noise function $f_n$ is defined as: $f_n(o_1)=o^n_1, f_n(o_2)=o^n_2$. Therefore, the probability in cluttered environments with $f_n$ is $p(o^n_1,r_1)= p(o^n_2,r_2)=0.5, p(o^n_1,r_2)=p(o^n_2,r_1)=0$.

Lets define a new noise function $f_{n_1}$, with $f_{n_1}(o_1)=o^n_2, f_{n_1}(o_2)=o^n_1$. Therefore, the probability in cluttered environments with $f_{n_1}$ is $p(o^n_2,r_1)= p(o^n_1,r_2)=0.5, p(o^n_1,r_1)=p(o^n_1,r_1)=0$.

As a result, it is clear that $f_{n_1}$ is homogeneous to $f_n$ without rewards, yet it is no longer homogeneous to $f_n$ with rewards.

\subsection{Self-Consistent Model-based Adaptation}
\label{appendix_subsec_scma}

Below we provide a detailed derivation of SCMA's adaptation loss. From Eq.~\ref{eq_adaptation_raw_1}, $\mathcal{L}_{\mathrm{KL}}$ leads to the following objective:
\begin{equation}
\label{eq_adaptation_raw_2}
    \resizebox{\linewidth}{!}{$
        \displaystyle{
            \mathbb{E}_{p(o^n_{1:T}|a_{1:T})}\mathbb{E}_{q(o_{1:T}|o^n_{1:T})} \Big [ -\log p(o_{1:T}|a_{1:T}) -\log q(o^n_{1:T}|o_{1:T}) \Big].
        }
    $}
\end{equation}

As mentioned in Sec.~\ref{subsec_adaptation}, $-\log p(o_{1:T}|a_{1:T})$ can be substituted with the evidence lower bound (ELBO) estimated by a world model pre-trained in the clean environment~\citep{hafner2019learning}:
\begin{equation*}
        \begin{aligned} 
            & \log p_{\scriptscriptstyle \mathrm{wm}}(o_{1:T}|a_{1:T}) = \log \int p_{\scriptscriptstyle \mathrm{wm}}(o_{1:T}, s_{1:T}|a_{1:T}) \\
            \geq &\sum\limits_{t=1}^T \mathbb{E}_{q_{\scriptscriptstyle \mathrm{wm}}(s_{1:T}|a_{1:T}, o_{1:T})}[\underbrace{\log p_{\scriptscriptstyle \mathrm{wm}}(o_t|s_{\leq t}, a_{<t})}_{\mathcal{J}_{o}^t} \\
            &\quad -\underbrace{\mathrm{D}_\mathrm{KL}(q_{\scriptscriptstyle \mathrm{wm}}(s_t|s_{<t},a_{<t}, o_t)\Vert p_{\scriptscriptstyle \mathrm{wm}}(s_t|s_{<t}, a_{<t})}_{\mathcal{J}_{kl}^t})].
    \end{aligned}
\end{equation*}

Leveraging the pre-trained world model, we can turn Eq.~\ref{eq_adaptation_raw_2} into the subsequent objective. For notation simplicity, we choose to omit the expectation notation $\mathbb{E}_{p(o^n_{1:T}|a_{1:T})}$:

\begin{equation}
\label{eq_scma}
    \resizebox{1.0\linewidth}{!}{$
    \displaystyle{
        \begin{aligned}
        &\mathbb{E}_{q(o_{1:T}|o^n_{1:T})} \Big [ -\log p(o_{1:T}|a_{1:T}) -\log q(o^n_{1:T}|o_{1:T}) \Big] \\
        \leq &\mathbb{E}_{q(o_{1:T}|o^n_{1:T})} \Big [ -\sum\limits_{t=1}^T \mathbb{E}_{q_{\scriptscriptstyle \mathrm{wm}}(s_{1:T}|a_{1:T}, o_{1:T})}\big [\log p_{\scriptscriptstyle \mathrm{wm}}(o_t|s_{\leq t}, a_{<t}) \\
        & \quad\quad + \mathrm{D}_\mathrm{KL}(q_{\scriptscriptstyle \mathrm{wm}}(s_t|s_{<t},a_{<t}, o_t)\Vert p_{\scriptscriptstyle \mathrm{wm}}(s_t|s_{<t}, a_{<t}))\big ] \\
        & \quad\quad - \log q(o^n_{1:T}|o_{1:T}) \Big] \\
        \stackrel{(*)}{\approx} & \mathbb{E}_{q(o_{1:T}|o^n_{1:T})} \Big [ -\sum\limits_{t=1}^T \mathbb{E}_{q_{\scriptscriptstyle \mathrm{wm}}(s_{1:T}|a_{1:T}, o_{1:T})}\big [\log p_{\scriptscriptstyle \mathrm{wm}}(o_t|s_{\leq t}, a_{<t})\big ] \\
        & \quad\quad -\sum\limits_{t=1}^T \log q(o^n_t|o_t)\Big ] \\
        = & \sum\limits_{t=1}^T - \underbrace{\mathbb{E}_{q(o_{1:T}|o^n_{1:T})} \mathbb{E}_{q_{\scriptscriptstyle \mathrm{wm}}(s_{1:T}|a_{1:T}, o_{1:T})}\big [\log p_{\scriptscriptstyle \mathrm{wm}}(o_t|s_{\leq t}, a_{<t})\big ]}_{\mathcal{L}_{sc}} \\
        & \quad\quad - \underbrace{\mathbb{E}_{q(o_{1:T}|o^n_{1:T})}[\log q(o^n_t|o_t)]}_{\mathcal{L}_{n}}.
    \end{aligned}
    }
$}
\end{equation}

In $(*)$, we choose to drop the KL-loss term $\mathrm{D}_\mathrm{KL}(q_{\scriptscriptstyle \mathrm{wm}}(s_t|s_{<t},a_{<t}, o_t) \Vert p_{\scriptscriptstyle \mathrm{wm}}(s_t|s_{<t}, a_{<t}))$. The goal of the KL-loss term is to promote policy optimization by enabling trajectory generation, which is unnecessary during adaptation as we do not modify the policy. Moreover, we empirically find that adding the KL-loss term has a negative impact during adaptation by harming the reconstruction results, which is consistent with the findings in previous works~\citep{higgins2017beta,chen2018isolating}.
% \zxn{more?}

\paragraph{Adaptation with Rewards} The above framework can be simply extended to consider rewards by changing $(o_t)$ to $(o_t,r_t)$. Specifically, we first redefine $\mathcal{L}_{\mathrm{KL}}$ as following:
\begin{equation*}
\begin{aligned}
    \mathcal{L}_{\mathrm{KL}} \coloneqq & \mathrm{D}_{\mathrm{KL}} \Big(p(o^n_{1:T},r_{1:T}|a_{1:T})q(o_{1:T}|o^n_{1:T}) \\
    &\quad\quad\quad\quad \big\Vert p(o_{1:T},r_{1:T}|a_{1:T})q(o^n_{1:T}|o_{1:T})\Big),
\end{aligned}
\end{equation*}
which leads to the subsequent simplified objective (similar to the derivation in Eq.~\ref{eq_adaptation_raw_1}):
\begin{equation}
\label{eq_adaptation_raw_2_r}
    \begin{aligned}
        & \mathbb{E}_{p(o^n_{1:T}, r_{1:T}|a_{1:T})}\mathbb{E}_{q(o_{1:T}|o^n_{1:T})} \Big [ -\log p(o_{1:T}, r_{1:T}|a_{1:T}) \\
        & \quad\quad\quad\quad\quad\quad\quad\quad\quad\quad\quad\quad\quad\quad - \log q(o^n_{1:T}|o_{1:T}) \Big].
    \end{aligned}
\end{equation}

We can also extend the world model's objective to consider rewards:

\begin{equation}
\label{eq_rssm_r}
    \resizebox{\linewidth}{!}{$
    \displaystyle{
        \begin{aligned}
        & \log p_{\scriptscriptstyle \mathrm{wm}}(o_{1:T}, r_{1:T}|a_{1:T}) = \log \int p_{\scriptscriptstyle \mathrm{wm}}(o_{1:T}, s_{1:T}|a_{1:T}) \mathrm{d} s_{1:T} \\
            \geq &\sum\limits_{t=1}^T \mathbb{E}_{q_{\scriptscriptstyle \mathrm{wm}}(s_{1:T}|a_{1:T}, o_{1:T})}[\underbrace{\log p_{\scriptscriptstyle \mathrm{wm}}(o_t|s_{\leq t}, a_{<t})}_{\mathcal{J}_{o}^t} + \underbrace{\log p_{\scriptscriptstyle \mathrm{wm}}(r_t|s_{\leq t}, a_{<t})}_{\mathcal{J}_{rew}^t} \\
            &\quad -\underbrace{\mathrm{D}_\mathrm{KL}(q_{\scriptscriptstyle \mathrm{wm}}(s_t|s_{<t},a_{<t}, o_t)\Vert p_{\scriptscriptstyle \mathrm{wm}}(s_t|s_{<t}, a_{<t})}_{\mathcal{J}_{kl}^t})],
    \end{aligned}
    }
$}
\end{equation}
where we include a reward model $ p_{\scriptscriptstyle \mathrm{wm}}(r_t|s_{\leq t}, a_{<t})$ and reward loss $\mathcal{J}_{rew}^t$ to predict reward signals.

Similar to the derivation in Eq.~\ref{eq_scma}, we can combine Eq.~\ref{eq_adaptation_raw_2_r} and Eq.~\ref{eq_rssm_r} to obtain adaptation objective with rewards (we again omit $\mathbb{E}_{p(o^n_{1:T}, r_{1:T}|a_{1:T})}$ for notation simplicity):
\begin{equation*}
\resizebox{\linewidth}{!}{$
    \displaystyle{
    \begin{aligned}
        & \mathbb{E}_{q(o_{1:T}|o^n_{1:T})} \Big [ -\log p(o_{1:T}, r_{1:T}|a_{1:T}) - \log q(o^n_{1:T}|o_{1:T}) \Big]\\
        \leq & \sum\limits_{t=1}^T -\underbrace{\mathbb{E}_{q(o_{1:T}|o^n_{1:T})} \mathbb{E}_{q_{\scriptscriptstyle \mathrm{wm}}(s_{1:T}|a_{1:T}, o_{1:T})}\big [\log p_{\scriptscriptstyle \mathrm{wm}}(o_t|s_{\leq t}, a_{<t})\big ]}_{\mathcal{L}_{sc}} \\
        & \quad\quad -\underbrace{\mathbb{E}_{q(o_{1:T}|o^n_{1:T})} \mathbb{E}_{q_{\scriptscriptstyle \mathrm{wm}}(s_{1:T}|a_{1:T}, o_{1:T})}\big [\log p_{\scriptscriptstyle \mathrm{wm}}(r_t|s_{\leq t}, a_{<t})\big ]}_{\mathcal{L}_{rew}} \\
        & \quad\quad - \underbrace{\mathbb{E}_{q(o_t|o^n_t)}[\log q(o^n_t|o_t)]}_{\mathcal{L}_{n}}.
    \end{aligned}
    }
$}
\end{equation*}

\section{Experimental Details}

\subsection{Visually Distracting Environments}

In this section, we provide an overview of the involved environments in our experiments, including DMControl-GB~\citep{Hansen2020SelfSupervisedPA,hansen2021generalization}, DMControlView~\citep{yang2024movie}, and RL-ViGen~\citep{yuan2024rl}.

\begin{figure}[H]
\centering
\captionsetup[subfigure]{justification=centering}
\subcaptionbox{\\Original Environment\label{fig_env_raw}}[0.19\linewidth]{\includegraphics[width=0.88\linewidth]{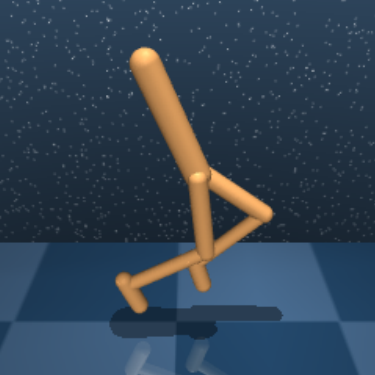}}
\subcaptionbox{$\mathrm{video\_hard}$ Environment\label{fig_env_video_hard}}[0.19\linewidth]{\includegraphics[width=0.88\linewidth]{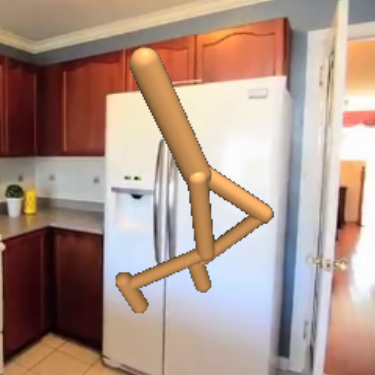}}
\subcaptionbox{$\mathrm{moving\_view}$ Environment\label{fig_env_moving}}[0.19\linewidth]{\includegraphics[width=0.88\linewidth]{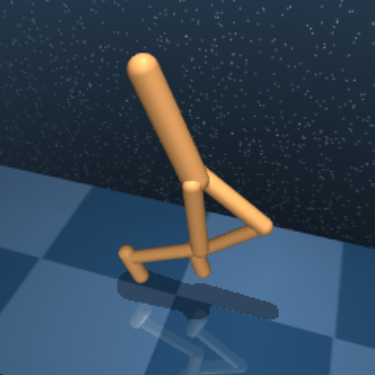}}
\subcaptionbox{$\mathrm{occlusion}$ Environment\label{fig_env_occlusion}}[0.19\linewidth]{\includegraphics[width=0.88\linewidth]{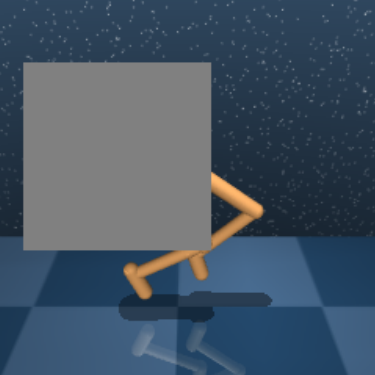}}
\subcaptionbox{$\mathrm{color\_hard}$ Environment\label{fig_env_color_hard}}[0.19\linewidth]{\includegraphics[width=0.88\linewidth]{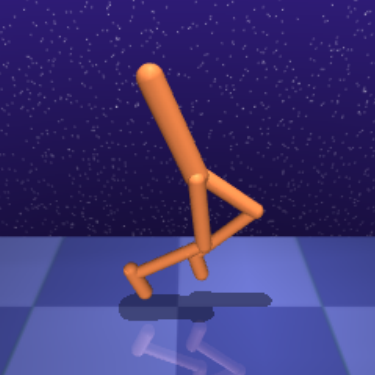}}

\caption{An overview of involved environments with DMControl. (a) The original DMControl (DeepMind Control Suite)~\citep{tassa2018deepmind} environment that we use to pre-train world models and policies. (b) The $\mathrm{video\_hard}$ environment~\citep{hansen2021generalization}, where the background is replaced by natural videos. (c) The $\mathrm{moving\_view}$ environment~\citep{yang2024movie} with moving camera views. (d) The $\mathrm{occlusion}$ environment where $1/4$ of each observation is randomly masked. (e) The $\mathrm{color\_hard}$ environment~\citep{hansen2021generalization}, where all objects are rendered with random colors.}
\label{fig_overview_dmc}
\end{figure}

\begin{figure}[H]
\centering
\captionsetup[subfigure]{justification=centering}
\subcaptionbox{\\Original Environment\label{fig_env_robo_raw}}[0.2\linewidth]{\includegraphics[width=0.9\linewidth]{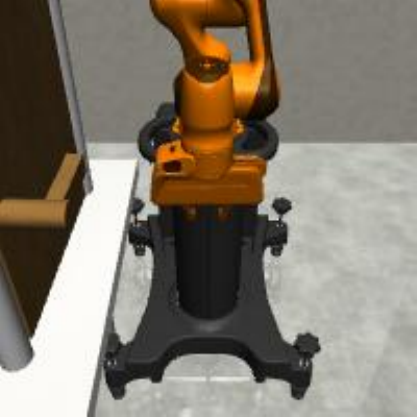}}
\hspace{2em}
% \subcaptionbox{Pre-training Environment\label{fig_env_robo_black}}[0.2\linewidth]{\includegraphics[width=0.9\linewidth]{imgs/robosuite-train-black.pdf}}
% \hspace{1em}
\subcaptionbox{$\mathrm{eval\_easy}$ Environment\label{fig_env_robo_easy}}[0.2\linewidth]{\includegraphics[width=0.9\linewidth]{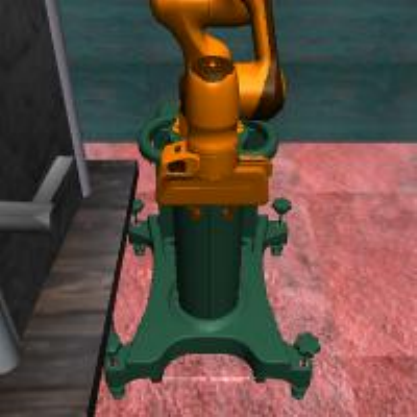}}
\hspace{2em}
\subcaptionbox{$\mathrm{eval\_extreme}$ Environment\label{fig_env_robo_extreme}}[0.2\linewidth]{\includegraphics[width=0.9\linewidth]{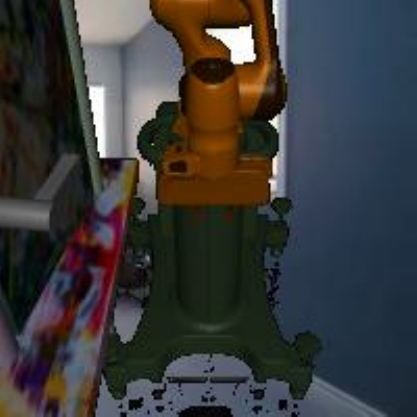}}

\caption{An overview of involved environments with Robosuite. (a) The original Robosuite environment~\citep{zhu2020robosuite} that we use to pre-train world models and policies. (b) The $\mathrm{eval\_easy}$ environment~\citep{yuan2024rl} includes changes in background appearance. (c) The $\mathrm{eval\_extreme}$ environment~\citep{yuan2024rl} employs a dynamic video background and includes varying lighting conditions.}
% \vspace{-1em}
\label{fig_overview_rlvigen}
\end{figure}

% \begingroup
% \setlength{\tabcolsep}{2.5pt}
% \begin{table*}[ht]
%     \centering
%     \scriptsize
%     \begin{tabular}{c|c|ccccccc}
%     \toprule
%     \text{RL-ViGen} & Oracle & SCMA(Ours) & SGQN & PIEG & SRM & SVEA & DrQ-v2 & CURL \\
%     \midrule
%     Door (easy) & \multirow{2}{*}{429$\pm$35} & \textbf{416 $\pm$ 26} & 391 $\pm$ 95  & 387 $\pm$ 53 & 337 $\pm$ 110 & 268 $\pm$ 136 & 4 $\pm$ 2 & 6 $\pm$ 5 \\
%     Door (extreme)  & & \textbf{380 $\pm$ 30} & 160 $\pm$ 122 & 87 $\pm$ 82  & 31 $\pm$ 18   & 62 $\pm$ 56   & 2 $\pm$ 1    & 2 $\pm$ 1 \\
%     \midrule
%     Lift (easy) & \multirow{2}{*}{168$\pm$55} & 19 $\pm$ 5   & 31 $\pm$ 17   & \textbf{96 $\pm$ 26}  & 69 $\pm$ 32   & 43 $\pm$ 18   & 1 $\pm$ 1    & 0 $\pm$ 0 \\ 
%     Lift (extreme) & & \textbf{15 $\pm$ 9}   & 7 $\pm$ 7     & 7 $\pm$ 5    & 0 $\pm$ 0     & 8 $\pm$ 5     & 1 $\pm$ 1    & 0 $\pm$ 0 \\
%     \midrule
%     TwoArmPegInHole (easy) & \multirow{2}{*}{475$\pm$3} & 340 $\pm$ 27 & 349 $\pm$ 23  & \textbf{430 $\pm$ 22} & 419 $\pm$ 45  & 414 $\pm$ 58  & 140 $\pm$ 13 & 150 $\pm$ 20 \\
%     TwoArmPegInHole (extreme) & & 227 $\pm$ 24 & 257 $\pm$ 31  & \textbf{332 $\pm$ 48} & 161 $\pm$ 27  & 155 $\pm$ 18  & 131 $\pm$ 6  & 147 $\pm$ 15 \\
%     \bottomrule
%     \end{tabular}
%     \vspace{0.5em}
%     \caption{Performance (mean $\pm$ std) on Table-top Manipulation tasks in RL-ViGen benchmark.}
%     \label{table_rlvigen}
% \end{table*}
% \endgroup
% \vspace{-2em}

\begin{table*}[tb]
\subcaptionbox{SCMA}[0.5\linewidth]{
\resizebox{0.9\linewidth}{!}{
    \tablestyle{1pt}{1.1}
    \centering
    \footnotesize
    \begin{tabular}{c|cccc}
            % & \multicolumn{4}{c}{SCMA} \\
        \toprule
        Task & $\mathrm{video\_hard}$ & $\mathrm{moving\_view}$ & $\mathrm{color\_hard}$ & $\mathrm{occlusion}$ \\
        \midrule
        $\mathrm{ball\_cup\text{-}catch}$ & 809 & 745 & 817 & 899 \\
        $\mathrm{cartpole\text{-}swingup}$ & 773 & 708 & 809 & 779 \\
        $\mathrm{finger\text{-}spin}$ & 948 & 952 & 965 & 920 \\
        $\mathrm{walker\text{-}stand}$ & 953 & 977 & 984 & 976 \\
        $\mathrm{walker\text{-}walk}$ & 722 & 922 & 954 & 902 \\
        \midrule
         Averaged & 841.0 & 860.8 & 905.8 & 895.2 \\
         \bottomrule
    \end{tabular}
}
}
\subcaptionbox{SCMA (w/o r)}[0.49\linewidth]{
\resizebox{0.9\linewidth}{!}{
    \tablestyle{1pt}{1.1}
    \centering
    \footnotesize
    \begin{tabular}{c|cccc}
            % & \multicolumn{4}{c}{SCMA} \\
        \toprule
        Task & $\mathrm{video\_hard}$ & $\mathrm{moving\_view}$ & $\mathrm{color\_hard}$ & $\mathrm{occlusion}$ \\
        \midrule
        $\mathrm{ball\_cup\text{-}catch}$ & 215 & 616 & 881 & 748 \\
        $\mathrm{cartpole\text{-}swingup}$ & 145 & 188 & 158 & 97\\
        $\mathrm{finger\text{-}spin}$ & 769 & 46 & 814 & 268 \\
        $\mathrm{walker\text{-}stand}$ & 328 & 929 & 745 & 297 \\
        $\mathrm{walker\text{-}walk}$ & 129 & 478 & 634 & 149 \\
        \midrule
         Averaged & 317.2 & 451.4 & 646.4 & 311.8 \\
         \bottomrule
    \end{tabular}
}
}
\\
\subcaptionbox{MoVie}[0.49\linewidth]{
\resizebox{0.9\linewidth}{!}{
    \tablestyle{1pt}{1.1}
    \centering
    \footnotesize
    \begin{tabular}{c|cccc}
            % & \multicolumn{4}{c}{SCMA} \\
        \toprule
        Task & $\mathrm{video\_hard}$ & $\mathrm{moving\_view}$ & $\mathrm{color\_hard}$ & $\mathrm{occlusion}$ \\
        \midrule
        $\mathrm{ball\_cup\text{-}catch}$ & 41 & 951 & 67 & 33\\
        $\mathrm{cartpole\text{-}swingup}$ & 83 & 196 & 102 & 120\\
        $\mathrm{finger\text{-}spin}$ & 2 & 896 & 652 & 1 \\
        $\mathrm{walker\text{-}stand}$ & 127 & 712 & 121 & 124\\
        $\mathrm{walker\text{-}walk}$ & 39 & 810 & 38 & 52 \\
        \midrule
         Averaged & 58.4 & 713.0 & 196.0 & 66.0 \\
         \bottomrule
    \end{tabular}
}
}
\subcaptionbox{PAD}[0.49\linewidth]{
\resizebox{0.9\linewidth}{!}{
    \tablestyle{1pt}{1.1}
    \centering
    \footnotesize
    \begin{tabular}{c|cccc}
            % & \multicolumn{4}{c}{SCMA} \\
        \toprule
        Task & $\mathrm{video\_hard}$ & $\mathrm{moving\_view}$ & $\mathrm{color\_hard}$ & $\mathrm{occlusion}$ \\
        \midrule
        $\mathrm{ball\_cup\text{-}catch}$ & 130 & 750 & 563 & 145 \\
        $\mathrm{cartpole\text{-}swingup}$ & 123 & 561 & 630 & 142 \\
        $\mathrm{finger\text{-}spin}$ & 96 & 603 & 803 & 15 \\
        $\mathrm{walker\text{-}stand}$ & 336 & 955 & 797 & 305 \\
        $\mathrm{walker\text{-}walk}$ & 108 & 645 & 468 & 94 \\
        \midrule
         Averaged & 158.6 & 702.8 & 652.2 & 140.2 \\
         \bottomrule
    \end{tabular}
}
}
\vspace{-1em}
\caption{We report the detailed performance of SCMA, SCMA(w/o r) and other adaptation-based baselines across $4$ distracting environments.}
\label{table_adaptation_all}
\end{table*}

\begin{figure*}[tb]
\centering
\centerline{\includegraphics[width=0.7\linewidth]{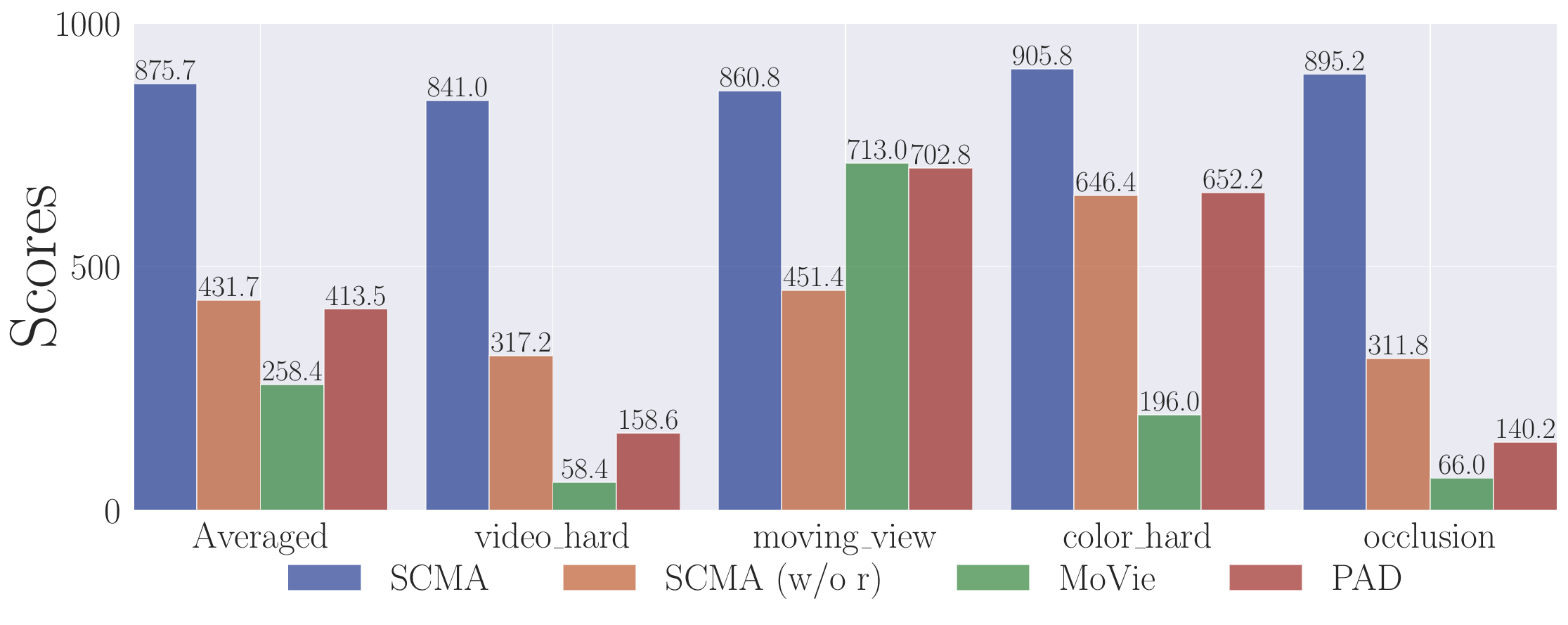}}
\vspace{-1em}
\caption{Average performance of SCMA, SCMA (w/o r), and other adaptation-based baselines across $4$ distracting environments. Our proposed method SCMA achieves the highest performance under distractions.} 
\label{fig_average_dmc}
\vspace{-1em}
\end{figure*}

\subsection{Quantitative Results}
\label{appendix_subsec_quantitative}
% In this section, we provide detailed training results of SCMA in both idealized training environments and distracting testing environments.
In this section, we provide detailed experimental results of SCMA. Unless otherwise stated, the result of each task is evaluated over $3$ seeds and we report the performance of the last episode. For table-top manipulation tasks, we report the performance of each trained agent by collecting $10$ trials on each scene ($100$ trials in total).

\subsubsection{Adaptation Results in DMControl}

We provide the performance curve of SCMA when adapting to distracting environments. For SCMA, the agent is trained in clean environments for $1$M timesteps and then adapts to visually distracting environments for another $0.1$M timesteps ($0.4$M for $\mathrm{video\_hard}$). It should be noted that although we adapt the agent to the $\mathrm{video\_hard}$ environments for $0.4$M steps, SCMA can achieve competitive results with only $10\%$ timesteps in most tasks, including $\mathrm{finger}\text{-}\mathrm{spin}$, $\mathrm{walker}\text{-}\mathrm{stand}$, $\mathrm{walker}\text{-}\mathrm{walk}$, as shown in Fig.~\ref{fig_appendix_video_hard}.

\begin{figure}[H]
    \centering
    \includegraphics[width=\linewidth]{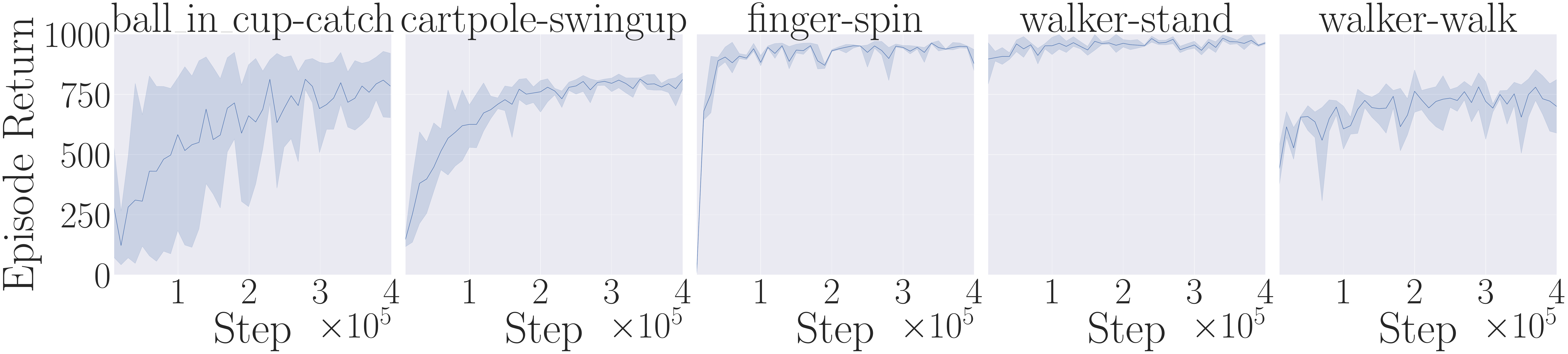}
    \caption{Adaptation performance of SCMA in the $\mathrm{video\_hard}$ environment.} 
    \label{fig_appendix_video_hard}
\end{figure}
\vspace{-2em}

\begin{figure}[H]
    \centering
    \includegraphics[width=\linewidth]{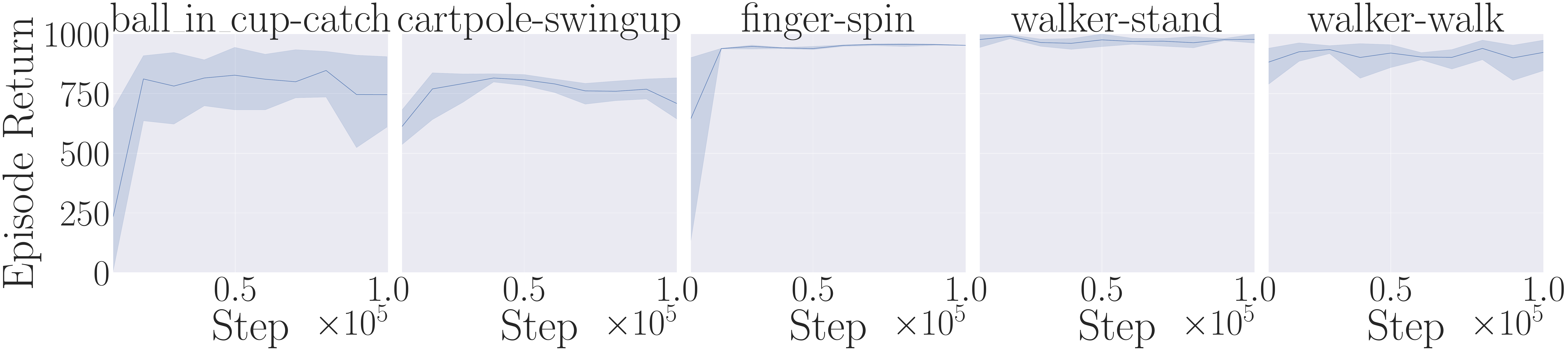}
    \caption{Adaptation performance of SCMA in the $\mathrm{moving\_view}$ environment.}
    \label{fig_appendix_moving_view}
\end{figure}
\vspace{-2em}

\begin{figure}[H]
    \centering
    \includegraphics[width=\linewidth]{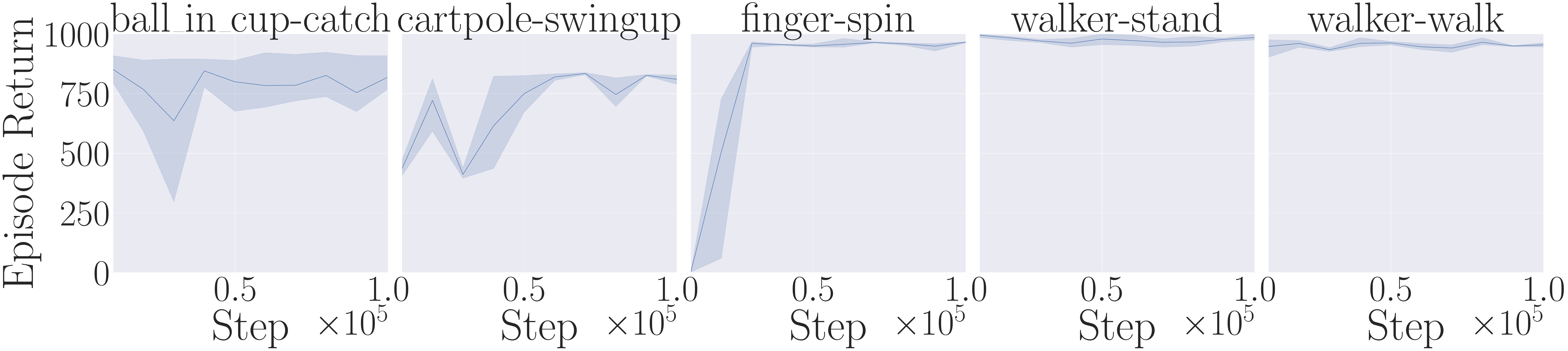}
    \caption{Adaptation performance of SCMA in the $\mathrm{color\_hard}$ environment.}
    \label{fig_appendix_color_hard}
\end{figure}
\vspace{-2em}

\begin{figure}[H]
    \centering
    \includegraphics[width=\linewidth]{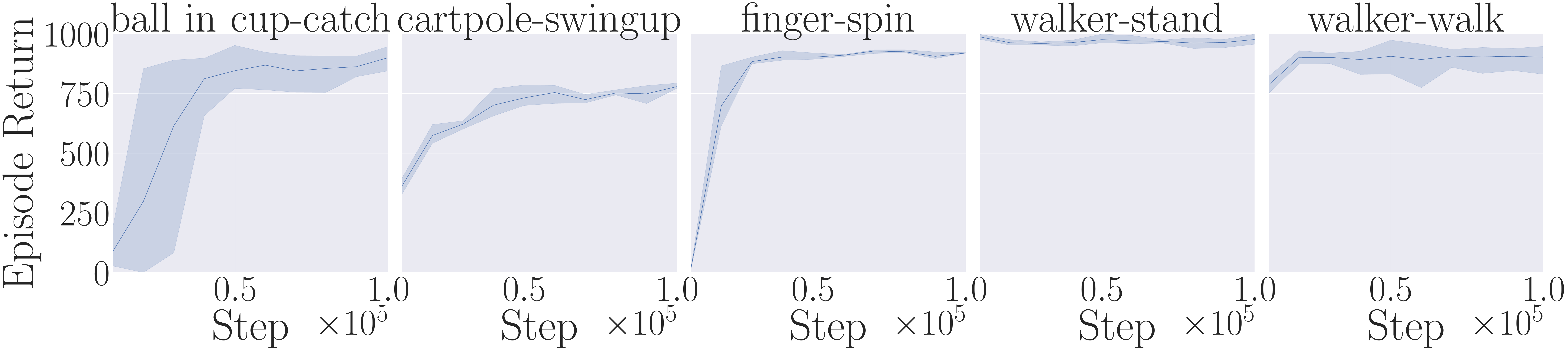}
    \caption{Adaptation performance of SCMA in the $\mathrm{occlusion}$ environment.}
    \label{fig_appendix_occlusion}
\end{figure}
\vspace{-2em}

\subsection{Adaptation without Rewards}

We report the detailed performance of SCMA, SCMA (w/o r), and other adaptation-based baselines in $4$ different distracting environments, where SCMA (w/o r) means removing $\mathcal{L}_{rew}$ during adaptation. From the detailed results presented in Table~\ref{table_adaptation_all} and average results presented in Fig.~\ref{fig_average_dmc}, we can see that SCMA obtains the best performance under distractions. Moreover, the results in Fig.~\ref{fig_average_dmc} show that even without rewards, SCMA (w/o r) still achieves the highest average performance compared to other adaptation-based baselines.

\subsubsection{Adaptation Results in RL-ViGen}

We present the adaptation curve of SCMA in RL-ViGen~\citep{yuan2024rl} in Fig.~\ref{fig_appendix_eval_easy} and Fig.~\ref{fig_appendix_eval_extreme}. For SCMA, the agent is trained in clean environments for $0.5$M timesteps and then adapts to visually distracting environments for another $0.5$M timesteps. Following~\citet{yuan2024rl}, we evaluate each trained agent with $10$ trails on each scene ($100$ trails in total) and report final results in Table.~\ref{table_rlvigen}. 

% \begin{figure}[H]
%     \centering
%     \includegraphics[width=\linewidth]{imgs/appendix_robosuite_oracle.pdf}
%     \caption{Pre-training performance of SCMA in training environments in Robosuite. We report the detailed pre-training results of SCMA in the clean training environments (see Fig.~\ref{fig_env_robo_black}).}
%     \label{fig_appendix_robosuite_oracle}
% \end{figure}
% \vspace{-1em}

\begin{figure}[H]
    \centering
    \includegraphics[width=\linewidth]{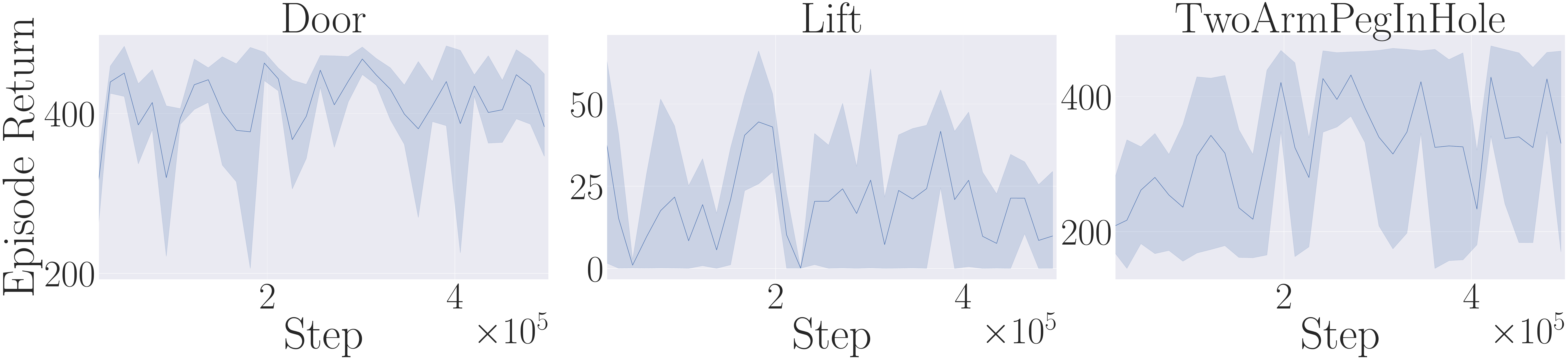}
    \caption{Adaptation performance of SCMA in the $\mathrm{eval\_easy}$ environment in RL-ViGen.} 
    \label{fig_appendix_eval_easy}
\end{figure}
\vspace{-1em}

\begin{figure}[H]
    \centering
    \includegraphics[width=\linewidth]{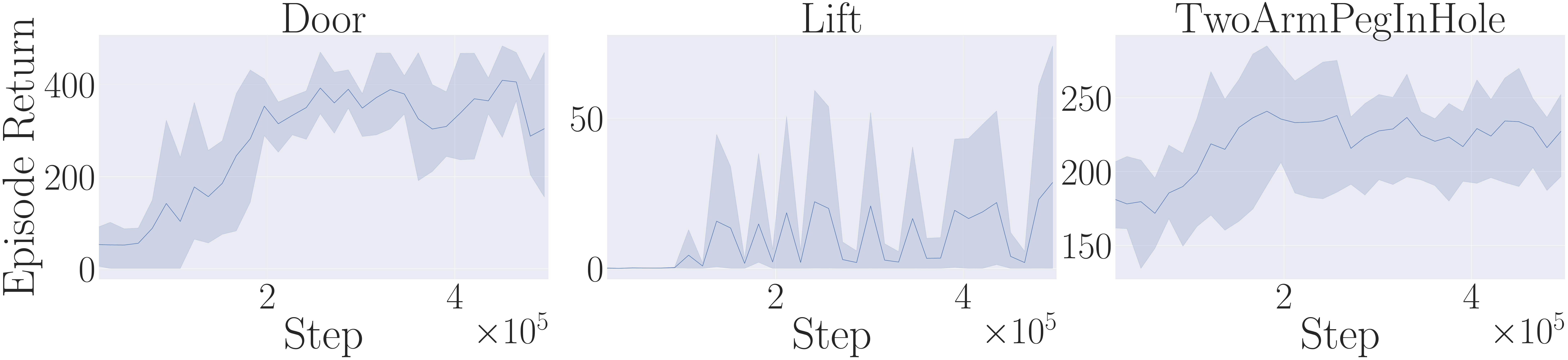}
    \caption{Adaptation performance of SCMA in the $\mathrm{eval\_extreme}$ environment in RL-ViGen.} 
    \label{fig_appendix_eval_extreme}
\end{figure}
\vspace{-1em}

\begin{figure*}[t]
\centering
\centerline{\includegraphics[width=1.0\linewidth]{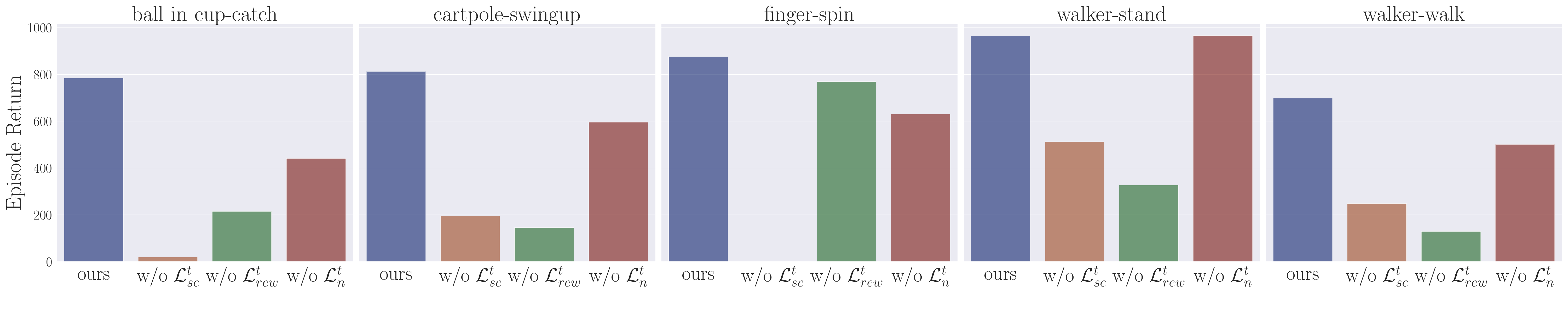}}
\vspace{-1em}
\caption{Ablation for different loss components' effects on the adaptation results in $\mathrm{video\_hard}$ environments. We separately remove the visual loss $\mathcal{L}^t_{visual}$, the reward prediction loss $\mathcal{L}^t_{rew}$ and mask penalty loss $\mathcal{L}^t_{reg}$ during adaptation. } 
\label{fig_different_loss}
\vspace{-1.5em}
\end{figure*}

\subsubsection{Zero-shot Generalization Performance}

We also investigate the zero-shot generalization performance of SCMA across different tasks in the $\mathrm{video\_hard}$ environment. Specifically, we separately optimize denoising models in $\mathrm{walker\text{-}}walk$ and $\mathrm{walker\text{-}stand}$ task. Then we take the denoising model adapted to one task and directly evaluate its zero-shot performance in another task. The results are presented in Table~\ref{table_cross_task} in the Appendix.

\begin{table}[h]
    \centering
    \footnotesize
\resizebox{0.9\linewidth}{!}{
    \begin{tabular}{c|cccc}
        \toprule
        Task & \multicolumn{2}{c}{$\mathrm{walker\text{-}stand}$} & \multicolumn{2}{c}{$\mathrm{walker\text{-}walk}$}\\
         \cmidrule(lr){2-3}\cmidrule(lr){4-5}
         Condition & In Domain & Transfer & In Domain & Transfer \\
         \midrule
         SCMA & 953\scriptsize{$\pm$4}& 956\scriptsize{$\pm$18} & 722\scriptsize{$\pm$89} & 652.14\scriptsize{$\pm$76} \\
         \bottomrule
    \end{tabular}
}
    \caption{In-domain and zero-shot generalization performance in $\mathrm{video\_hard}$ environment. We take denoising models trained in $\mathrm{walker\text{-}stand}$ and $\mathrm{walker\text{-}walk}$ and report their zero-shot generalization results evaluated in $\mathrm{walker\text{-}walk}$ and $\mathrm{walker\text{-}stand}$ separately (labeled as \textit{transfer}).}
    \vspace{-1.5em}
    \label{table_cross_task}
\end{table}

\subsubsection{Wall Clock Time Report}
\label{sec_wall_time}

Although we report the performance in the $\mathrm{video\_hard}$ environment after $0.4$M adaptation steps for best results, SCMA can usually achieve competitive results within much fewer steps. To demonstrate this idea, we report the wall clock time and adaptation episode for SCMA to reach $90\%$ of the final performance in the $\mathrm{video\_hard}$ environments. All experiments are conducted with NVIDIA GeForce RTX 4090 and Intel(R) Xeon(R) Gold 6330 CPU. 

\begingroup
\begin{table}[ht]
\tablestyle{0.8pt}{1.0}
    \centering
    \scriptsize
    \begin{tabular}{c|ccccc}
        \toprule
         Time/Episode &  ball\_in\_cup-catch & cartpole-swingup & finger-spin & walker-stand & walker-walk\\
         \midrule
         SCMA & 6.6h/180 & 6.1h/170 & 1.3h/40 & 0.17h/10 & 1.3h/40\\
         % SCMA(aug) & 8.9h/360 & 12.9h/510 & 5.0h/180 & 1.4h/60 & 2.1h/90\\
         \bottomrule
    \end{tabular}
    \vspace{-0.5em}
    \caption{Wall clock time for SCMA to reach $90\%$ of the final performance in the $\mathrm{video\_hard}$ environments.}
    \vspace{-1.5em}
    \label{tabale_video_hard_wall_time}
\end{table}
\endgroup
% \vspace{-2em}

From the Table above, we can see that SCMA only needs approximately $10\%$ of total adaptation time-steps to obtain a competitive performance for most tasks. Moreover, SCMA is a policy-agnostic method. Therefore, it can naturally utilize existing offline datasets to promote adaptation, and thus further alleviate the need to interact with downstream distracting environment.

\subsubsection{Ablation Results for Different Loss Components}

We also provide a detailed ablation on how different loss components in SCMA affect the final adaptation performance in $\mathrm{video\_hard}$ environments. We separately removed the $3$ loss components from SCMA during adaptation, namely self-consistent reconstruction loss $\mathcal{L}^t_{sc}$, reward prediction loss $\mathcal{L}^t_{rew}$, and noisy reconstruction loss $\mathcal{L}^t_{n}$. The results are presented in Fig.~\ref{fig_different_loss}.

\subsection{Real-world Robot Data}
\label{appendix_subsec_robot_score}
We report the detailed performance of SCMA on real-world robot data in Table~\ref{table_robot_scma}. The goal of the inverse dynamics model (IDM) is to predict the intermediate action $a_t$ based on observations $(o_t,o_{t+1})$. To verify SCMA's effectiveness on real-world robot data, we first pre-train the IDM and world model with data collected in the $\mathrm{train}$ setting. Then we optimize the denoising model in distracting settings and compare the action prediction accuracy of the IDM when using cluttered observations (labeled as IDM) versus using the outputs of the denoising model (labeled as IDM{\scriptsize{$+$SCMA}}).

\begin{table}[h]
    \centering
    \footnotesize
    \begin{tabular}{c|cc}
        \toprule
        Settings & IDM & IDM\scriptsize{$+$SCMA} \\
        \midrule
        $\mathrm{train}$ & 2.86 & - \\
        $\mathrm{fruit\_bg}$ & $3.75$ & $3.52$ \\
        $\mathrm{color\_bg}$ & $3.83$ & $3.73$ \\
        $\mathrm{varying\_light}$ & $3.22$ & $3.10$ \\
        \bottomrule
    \end{tabular}
    \caption{The Mean Squared Error (MSE) action prediction error of the IDM under different settings.}
    \vspace{-1em}
    \label{table_robot_scma}
\end{table}

% \begin{figure}[H]
% \centering
% \begin{subfigure}[b]{\linewidth}
%     \centering
%     \includegraphics[width=0.95\linewidth]{imgs/appendix_ablation_1.pdf}
% \end{subfigure}

% \begin{subfigure}[b]{0.72\linewidth}
%     \centering
%     \includegraphics[width=\linewidth]{imgs/appendix_ablation_2.pdf}
% \end{subfigure}

% % \vspace{-1.25em}
% \caption{Ablation for different loss components' effect on the adaptation results. We separately remove the visual loss $\mathcal{L}^t_{visual}$, the reward prediction loss $\mathcal{L}^t_{rew}$ and mask penalty loss $\mathcal{L}^t_{reg}$ and report their effect on the adaptation results in $\mathrm{video\_hard}$ environments.}
% \label{fig_appendix_ablation}
% \end{figure}

\subsection{Qualitative Results}
\label{appendix_qualitative}

\subsubsection{Visualization of Adaptation Results in Visually Distracting Environments}
In this section, we provide the visualization of SCMA's adaptation results in different visually distracting environments. We visualize the environment's raw observation as well as the outputs of the denoising model in Fig.~\ref{appendix_fig_visualize_scma_dmc}.

\begin{figure}[H]
\centering
\subcaptionbox{$\mathrm{video\_hard}$\label{fig_visualization_video_hard}}[0.8\linewidth]{
    \centering
    \includegraphics[width=0.7\linewidth]{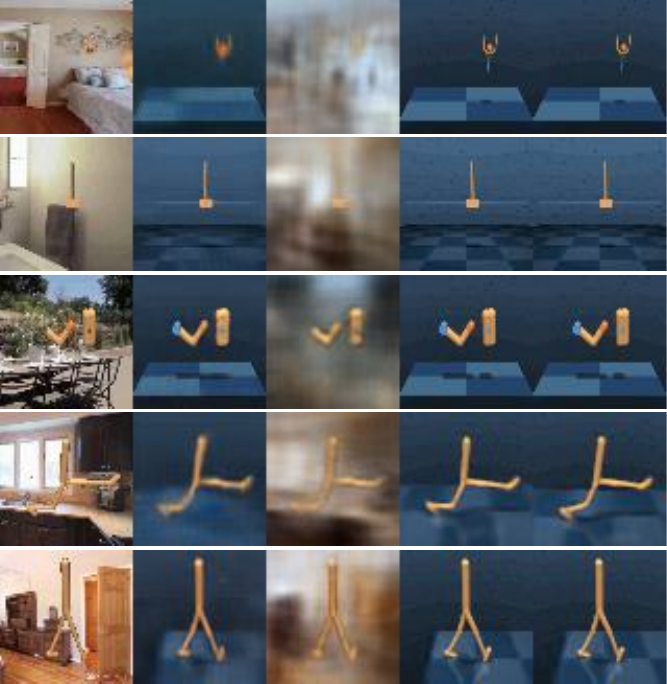}
}
% \\
\end{figure}
\vspace{-2em}
\begin{figure}[H]\ContinuedFloat
\centering
\subcaptionbox{$\mathrm{moving\_view}$\label{fig_visualization_moving_view}}[0.8\linewidth]{
    \centering
    \includegraphics[width=0.7\linewidth]{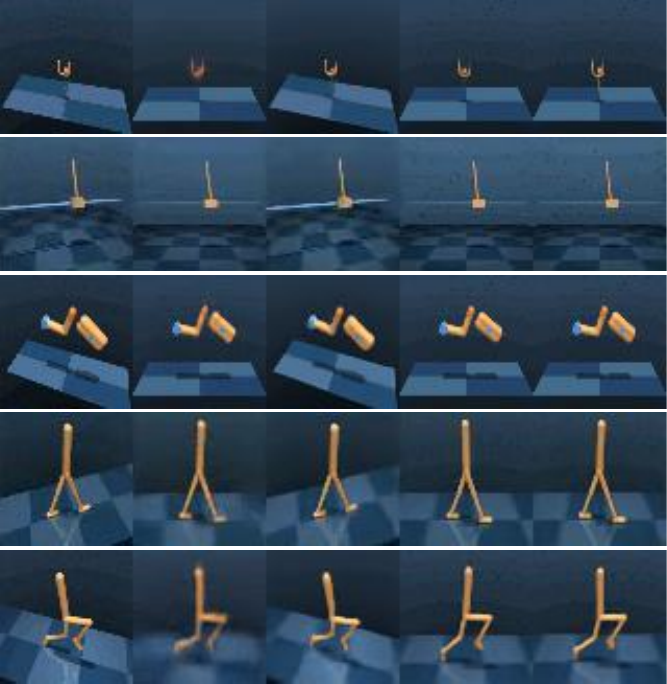}
}
% \vspace{0.5em}
\\
\subcaptionbox{$\mathrm{color\_hard}$\label{fig_visualization_color_hard}}[0.8\linewidth]{
    \centering
    \includegraphics[width=0.7\linewidth]{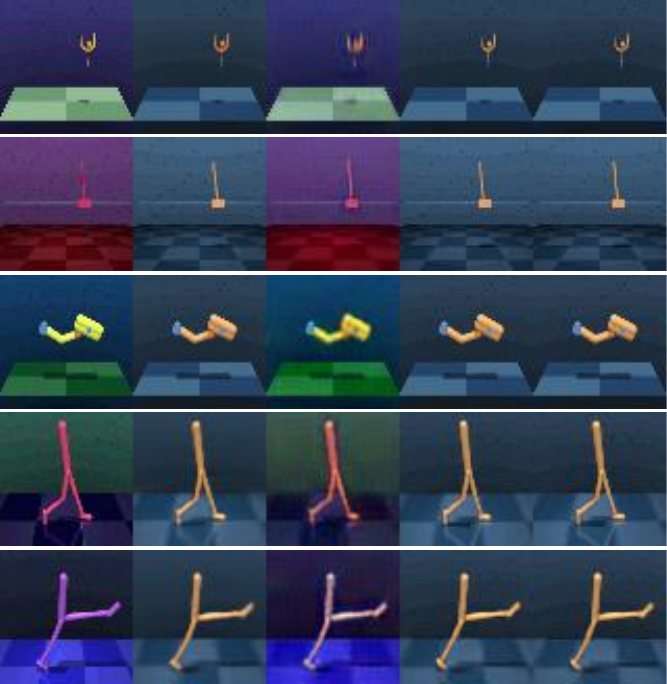}
}
\\
% \vspace{0.5em}
\subcaptionbox{$\mathrm{occlusionw}$\label{fig_visualization_occlusion}}[0.8\linewidth]{
    \centering
    \includegraphics[width=0.7\linewidth]{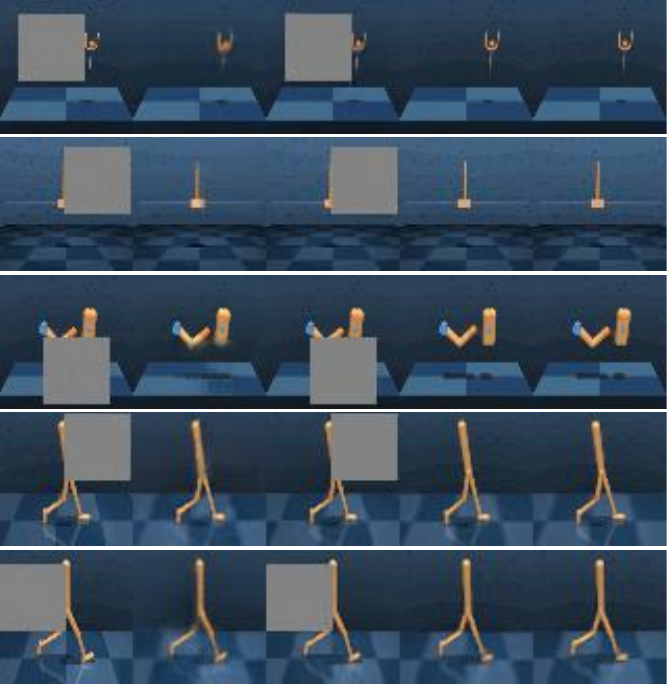}
}
\caption{Visualization of SCMA in different distracting environments. The columns from left to right separately represent (1) cluttered observations (2) outputs of the denoising model $m_{\mathrm{de}}$ (3) outputs of the noisy model $m_{\mathrm{n}}$ (4,5) prior and posterior reconstruction results by world models.}
\label{appendix_fig_visualize_scma_dmc}
\end{figure}

\section{Training Details}
For better reproductivity, we report all the training details of Sec.~\ref{sec_exp}, including the design choice of the denoising model, baseline implementations, and the hyper-parameters of SCMA.

\subsection{Baselines}
\label{appendix_baseline}
To evaluate the generalization capability of SCMA, we compare it to a variety of baselines in visually distracting environments. We will now introduce how different baselines are implemented and evaluated in each setting.

\paragraph{PAD~\citep{Hansen2020SelfSupervisedPA}}: PAD uses surrogate tasks to fine-tune the policy's representation to promote adaptation, such as image rotation prediction, and action prediction. The code follows \url{https://github.com/nicklashansen/dmcontrol-generalization-benchmark}.

\paragraph{MoVie~\citep{yang2024movie}}: MoVie incorporates spatial transformer networks (STN~\citep{jaderberg2015spatial}) to fill the performance gap caused by varying camera views. The code follows \url{https://github.com/yangsizhe/MoVie}.

\paragraph{SGQN~\citep{Bertoin2022LookWY}}: SGQN improves the generalization capability of RL agents by introducing a surrogate loss that regularizes the agent to focus on important pixels. The code follows \url{https://github.com/SuReLI/SGQN}.

\paragraph{TIA~\citep{fu2021learning}}: TIA learns a structured representation that separates task-relevant features from irrelevant ones. The code follows \url{https://github.com/kyonofx/tia}.

\paragraph{DreamerPro~\citep{deng2022dreamerpro}}: DreamerPro utilizes prototypical representation learning~\citep{caron2020unsupervised} to create representation invariant to distractions. The code follows \url{https://github.com/fdeng18/dreamer-pro}.

\paragraph{TPC~\citep{nguyen2021temporal}}: TPC improves the performance under distractions by forcing the representation to capture temporal predictable features. The code follows a newer version of TPC with higher results implemented in \url{https://github.com/fdeng18/dreamer-pro}.

For baselines implemented by us, their scores are taken from the paper if the evaluation setting is the same. Otherwise, their scores are estimated with our implementation. For other baselines, their scores are directly taken from the papers. It should be noted that although TIA, TPC, and DreamerPro were evaluated in environments with distracting video backgrounds in the first place, the original implementation uses a different video source. Therefore, we implement their algorithm with the official code while modifying the environment to use the same video source as $\mathrm{video\_hard}$ from DMControl-GB~\citep{hansen2021generalization}. 

\paragraph{Distracting Environments}: For $\mathrm{video}\_\mathrm{hard}$ and $\mathrm{color}\_\mathrm{hard}$ environment, the settings follow DMControl-GB~\citep{hansen2021generalization}. For $\mathrm{moving}\_\mathrm{view}$, the setting follows DMControl-View in MoVie~\citep{yang2024movie}. For $\mathrm{occlusion}$, we randomly cover $1/4$ of each observation with a grey rectangle. For evaluations in RL-ViGen~\citep{yuan2024rl}, the setting follows the original implementation.

% For baselines that we implement, We provide their detailed performance in the $\mathrm{video\_hard}$ environments in Fig.~\ref{fig_appendix_baseline}.
% \begin{figure}[H]
% \centering

% \begin{subfigure}[b]{\linewidth}
%     \centering
%     \includegraphics[width=0.95\linewidth]{imgs/appendix_baseline_1.pdf}
% \end{subfigure}

% \begin{subfigure}[b]{0.72\linewidth}
%     \centering
%     \includegraphics[width=\linewidth]{imgs/appendix_baseline_2.pdf}
% \end{subfigure}

% % \vspace{-1.25em}
% \caption{Performance of implemented baselines evaluated in the $\mathrm{video\_hard}$ environments. We report the detailed performance of implemented baselines in the $\mathrm{video\_hard}$ environments. Each task is evaluated over $3$ seeds.}
% \label{fig_appendix_baseline}
% \end{figure}

\subsection{Implementation Details}
\label{appendix_subsec_implementation}
\paragraph{Implementation of the Denoising Model}

The goal of the denoising model is to transfer cluttered observations to corresponding clean observations. Therefore, we implement the denoising model as a Resnet-based generator~\citep{Zhu2017UnpairedIT}, which is a generic image-to-image model. However, as mentioned in Sec.~\ref{subsec_homo}, we can encode some inductive bias in the denoising model’s architecture to handle specific types of distractions. In RL-ViGen, we consider two specific architectures of the denoising model: 1) mask model $m_{\mathrm{mask}}: \mathbb{R}^{h\times w\times c}\mapsto[0,1]^{h\times w\times c}$ to handle background distractions. 2) bias model $m_{\mathrm{bias}}:\mathbb{R}^{h\times w\times c}\mapsto\mathbb{R}^{h\times w\times 1}$ to handle lighting changes. The final denoise output is thus $m_{\mathrm{mask}}(o^n_t) \cdot o^n_t + m_{\mathrm{bias}}(o^n_t)$.

% \zxn{Appendix need to talk about practical implementation stop gradient.}

\subsection{Real-world Robot Data}
\label{appendix_subsec_robot_setting}
To verify the effectiveness of SCMA on real-world robot data. We manually collect real-world data with a Mobile ALOHA robot by performing an apple-grasping task. Specifically, we use the right gripper to grasp the apple and then put it in the target location. We record images captured by a front camera and $14$ joint poses, where the latter is the expected output of the inverse dynamics model (IDM).

We collect trajectories under $1$ normal and $3$ distracting settings: 1) $\mathrm{train}$: the normal setting with minimum distractions. 2) $\mathrm{fruit\_bg}$: various fruits are placed in the background. 3) $\mathrm{color\_bg}$: the scene is disrupted by a blue light. 4) $\mathrm{varying\_light}$: the lighting condition is modified. In the $\mathrm{train}$ setting, we first collect $20$ apple-grasping trajectories. Moreover, since IDM can be trained with trajectories collected with any policies, we additionally collect $50$ trajectories in the $\mathrm{train}$ setting with a random policy. Then we collect $10$ apple-grasping trajectories in each distracting setting. We provide visualization of each setting in Fig.~\ref{fig_robot_visualization}. During trajectory collection, data was recorded at a frame rate of $30$ fps, with each trajectory consisting of approximately $900$ to $1000$ frames. The quantitative results on real-world robot data are provided in Sec.~\ref{appendix_subsec_robot_score}.

\subsection{Hyper-parameters}
\label{sec_parameters}
\begin{table}[H]
    \centering
    \small
    \begin{tabular}{c|c}
    \toprule
    \textbf{Hyper-parameters}   &  Value \\
    \hline
    optimizer & adam \\
    adam\_epsilon & $1e^{-7}$ \\
    batch\_size & $55$ \\
    cnn\_activation\_function & $\mathrm{relu}$ \\
    collect\_interval & $100$ \\
    dense\_activation\_function & $\mathrm{elu}$ \\
    experience\_size & $1e^{6}$ \\
    grad\_clip\_norm & $100$ \\
    max\_episode\_length & $1000$ \\
    steps  &  $1e^{6}$ \\
    observation\_size & $64$ \\
    \hline
    \textbf{World Model} & \\
    \hline
    belief\_size & $200$ \\
    embedding\_size & $1024$ \\
    hidden\_size & $200$ \\
    model\_lr & $1e^{-3}$ \\
    \hline
    \textbf{Actor-Critic} & \\
    \hline
    actor\_lr & $8e^{-5}$ \\
    gamma & $0.99$ \\
    lambda & $0.95$ \\
    planning\_horizon & $15$ \\
    value\_lr & $8e^{-5}$ \\
    \hline
    \textbf{denoising model} & \\
    \hline
    denoise\_lr & $1e^{-4}$ \\
    denoise\_embedding\_size & $1024$ \\
    \bottomrule
    \end{tabular}
    \vspace{0.5em}
    \caption{Details of hyper-parameters.}
    \label{appendix_table_hyperpara}
\end{table}

The hyper-parameters of baselines from official implementations are taken from their implementations (see Appendix~\ref{appendix_baseline} above). SCMA is implemented based on a wildly adopted Dreamer~\citep{hafner2019dream} repository \url{https://github.com/yusukeurakami/dreamer-pytorch} and inherits the hyper-parameters. For completeness, we still list all hyper-parameters including inherited ones in Table~\ref{appendix_table_hyperpara}. \textbf{We also provide codes in the supplementary materials.}

\subsection{Algorithm}
\label{appendix_subsec_code}
We provide the pseudo-code of SCMA below:
\begin{algorithm}[H]
   \caption{Self-consistent Model-based Adaptation}
   \label{alg_SCMA}
\begin{algorithmic}
   \STATE {\bfseries Input:} Pre-trained world model model $p_{\scriptscriptstyle{\mathrm{wm}}}, q_{\scriptscriptstyle\mathrm{wm}}$, pre-trained policy $\pi$, denoising model $m_{\mathrm{de}}$ (denoted as $m_{\theta}$), noisy model $m_{\mathrm{n}}$ (denoted as $m_{\phi}$), distracting environment $\mathrm{Env}$, replay buffer $\mathcal{B}$, time horizon $H$, step-size $\eta$.
   \STATE {\bfseries Output:} Optimized denoising model $m_{\mathrm{de}}$.
    \FOR{\textit{each iteration}}
    \FOR{\textit{each update step}}
        \STATE {\small $\mathrm{//\, Sample\ a\ mini\text{-}batch\ from\ the\ buffer.}$}
        \STATE $\{o^n_{i}, a_{i}, r_{i}\}_{i:i+H}\sim\mathcal{B}$.
        \STATE {\small $\mathrm{//\, Optimize\ }m_{\mathrm{de}}\mathrm{\ and\ }m_{\mathrm{n}}\mathrm{\ with\ Eq.}$~\ref{eq_scma_loss}.}
        \STATE $\theta \leftarrow\theta - \eta \nabla_{\theta}\mathcal{L}_{\mathrm{SCMA}}$.
        \STATE $\phi \leftarrow\phi - \eta \nabla_{\phi}\mathcal{L}_{\mathrm{SCMA}}$.
    \ENDFOR
    \FOR{\textit{each collection step}}
        \STATE {\small $\mathrm{//\ Sample\ action\ with\ pocliy\ and\ denoising\ model.}$}
        \STATE $a_t\sim \pi (\cdot | m_{\theta}(o^n_t))$.
        \STATE {\small $\mathrm{//\ Interaction\ with\ the\ distracting\ environment.}$}
        \STATE $\{o^n_{t+1}, r_{t+1}\}\sim\mathrm{Env}(o_t^n,a_t)$
        \STATE {\small $\mathrm{// store\ data\ to\ the\ replay\ buffer.}$}
        \STATE $\mathcal{B}\leftarrow\mathcal{B}\cup\{o^n_t,a_t,o^n_{t+1},r_{t+1} \}$. 
    \ENDFOR
    \ENDFOR
\end{algorithmic}
\end{algorithm}

\end{document}